%% file: main.tex
\documentclass{article} 

\usepackage{geometry}
\usepackage{natbib}

\usepackage[utf8]{inputenc} 
\usepackage[T1]{fontenc}    

\usepackage{titletoc}
\usepackage[toc, page, header]{appendix} 

\usepackage[colorlinks=true, linkcolor=blue, citecolor=blue,urlcolor=black]{hyperref}
\usepackage{url}            
\usepackage{booktabs}       
\usepackage{amsfonts}       
\usepackage{nicefrac}       
\usepackage{microtype}      
\usepackage{xcolor}         
\usepackage{footnote}
\usepackage{graphicx}
\usepackage{subcaption}
\usepackage{booktabs} 

\usepackage{multirow}
\usepackage{colortbl}
\usepackage{arydshln}
\usepackage{amsmath}
\usepackage{amsthm}
\usepackage{amssymb}
\usepackage{mathtools}

\usepackage{comment}
\usepackage{bm}
\usepackage{bbm}
\usepackage{enumitem}
\setitemize{leftmargin=*}
\setenumerate{leftmargin=*}

\usepackage{algorithm}
\usepackage[noend]{algorithmic}

\allowdisplaybreaks

\renewcommand{\tilde}{\widetilde}
\renewcommand{\hat}{\widehat}
\renewcommand{\bar}{\overline}

\usepackage{cleveref}
\newtheorem{theorem}{Theorem}[section]
\newtheorem{lemma}[theorem]{Lemma}

\newtheorem{assumption}[theorem]{Assumption}

\Crefname{ALC@line}{Line}{Lines}
\Crefname{assumption}{Assumption}{Assumptions}
\Crefformat{equation}{Eq. #2(#1)#3}
\Crefrangeformat{equation}{Eqs. #3(#1)#4 to #5(#2)#6}
\Crefmultiformat{equation}{Eqs. #2(#1)#3}{ and #2(#1)#3}{, #2(#1)#3}{ and #2(#1)#3}
\Crefrangemultiformat{equation}{Eqs. #3(#1)#4 to #5(#2)#6}{ and #3(#1)#4 to #5(#2)#6}{, #3(#1)#4 to #5(#2)#6}{ and #3(#1)#4 to #5(#2)#6}

\input{math_commands}

\newcommand{\pll}{\kern 0.56em/\kern -0.8em /\kern 0.56em}

\newcommand{\dif}{\textnormal{d}}

\usepackage{pifont}

\title{
Finite-Time Convergence Analysis of ODE-based Generative Models for Stochastic Interpolants
}

\author{
Yuhao Liu \thanks{IIIS, Tsinghua University. Email: \texttt{\{liuyuhao21, hu-r24, chenyu23\}@mails.tsinghua.edu.cn}.}
\and
Rui Hu \footnotemark[1]
\and
Yu Chen \footnotemark[1]
\and
Longbo Huang \thanks{Corresponding author. IIIS, Tsinghua University. Email: \texttt{longbohuang@tsinghua.edu.cn}. }
}
\date{}

\usepackage{CJKutf8}

\begin{document}
\maketitle
\begin{CJK}{UTF8}{gbsn}

\begin{abstract}
Stochastic interpolants offer a robust framework for continuously transforming samples between arbitrary data distributions, holding significant promise for generative modeling. Despite their potential, rigorous finite-time convergence guarantees for practical numerical schemes remain largely unexplored. In this work, we address the finite-time convergence analysis of numerical implementations for ordinary differential equations (ODEs) derived from stochastic interpolants. Specifically, we establish novel finite-time error bounds in total variation distance for two widely used numerical integrators: the first-order forward Euler method and the second-order Heun's method. Furthermore, our analysis on the iteration complexity of specific stochastic interpolant constructions provides optimized schedules to enhance computational efficiency. Our theoretical findings are corroborated by numerical experiments, which validate the derived error bounds and complexity analyses. 
\end{abstract}

\input{contents/introduction}
\input{contents/background}

\input{contents/results-1}
\input{contents/results-2}
\input{contents/experiments}
\input{contents/conclusion}

\bibliographystyle{apalike}
\bibliography{ref}

\newpage
\appendix
\appendixpage

\startcontents[section]
\printcontents[section]{l}{1}{\setcounter{tocdepth}{2}}
\newpage

\input{contents/appendix-lemmas}
\input{contents/appendix-euler}
\input{contents/appendix-heun}

\end{CJK}
\end{document}

%% file: math_commands.tex
\def\eqref#1{equation~\ref{#1}}

\def\1{\bm{1}}

\DeclareMathAlphabet{\mathsfit}{\encodingdefault}{\sfdefault}{m}{sl}
\SetMathAlphabet{\mathsfit}{bold}{\encodingdefault}{\sfdefault}{bx}{n}

\newif\ifsup\supfalse
\suptrue

%% file: contents/introduction.tex
\section{Introduction}

Stochastic interpolants \citep{albergo2023building,albergo2023interpolant} provide a powerful framework for constructing generative models by learning deterministic or stochastic transformations that continuously map samples from an initial distribution $\rho_0$ to a target distribution $\rho_1$, governed by ordinary or stochastic differential equations (ODEs/SDEs).
To determine the differential equation, the approach constructs stochastic interpolations between $\rho_0$ and $\rho_1$ samples, then estimates a mean velocity field from these paths. With a learned approximation of the velocity field, one can build a generative model by solving the approximated ODE or SDE.
This novel framework unifies flow matching \citep{lipman2023flow} and score-based diffusion \citep{song2020improved,ho2020ddpm,song2021scorebased}, offering significant design flexibility through its various choices of the initial distribution and generalized interpolation formulation, making it an important subject for theoretical investigation.

Theoretical guarantees regarding generation error bounds have been established for both ODE-based and SDE-based transformations within the stochastic interpolant framework \citep{albergo2023building,albergo2023interpolant,benton2024error}. However, these existing analyses mainly examine the continuous-time scenario, relying on the idealized assumption of exact solutions. In practical applications, numerical approximation methods must be employed to solve these equations, introducing additional discretization errors. To develop a complete theoretical understanding of the computational complexity in the generation process, a rigorous analysis of finite-time convergence for discrete-time implementations becomes essential. In this paper, we focus on the ODE-based transformations, and investigate the following research question:

\begin{center}
    \textbf{What are the non-asymptotic convergence rates for discrete-time ODE implementations of stochastic interpolants? }
\end{center}

This problem is important yet technically challenging. While similar analyses exist for diffusion models and related approaches that transform Gaussian to target distributions (e.g., \citealt{li2025unified, li2024faster, huang2025convergence}), they cannot be directly extended to general stochastic interpolants due to the general data-to-data transformation structures. \cite{liu2025finitetime} recently established finite-time error bounds for the SDE case; however, their bounds exhibit singular behavior when the SDE degenerates to an ODE, indicating the necessity for novel analytical approaches to establish a complete theoretical foundation. 

In this paper, We derive rigorous finite-time error bounds for both first-order and second-order numerical methods in the ODE setting. Building upon an existing framework for total variation (TV) error estimation (Lemma \ref{lem:tv-dist}), we introduce several key innovations to address the challenges: (i) novel continuous-time interpolations carefully designed for discrete-time schemes, (ii) improved error decomposition that enables tighter bounds, and (iii) control of higher-order derivatives through refined analytical techniques.




\paragraph{Contributions} This work provides the first systematic analysis for discretized ODE implementations of stochastic interpolants, with three main theoretical contributions:

\begin{itemize}
    \item This work establishes the first finite-time error bounds in total variation (TV) distance for discrete-time numerical approximations of stochastic interpolant ODEs. We provide a complete theoretical characterization of both first-order (forward Euler) and second-order (Heun's) methods, rigorously quantifying the dependence of distributional approximation error on (i) problem parameters (including dimension $d$ and properties of initial/target distributions) and (ii) numerical parameters (particularly the choice of step sizes and the order of numerical scheme).

    \item Compared to existing analyses of ODE-based flows, our approach employs a novel error partition technique that yield tight bounds under an uncommon yet reasonable regularity assumptions. By introducing a Lipschitz divergence assumption that is typically satisfied for the true velocity functions in practical settings, we can decompose the total error into tractable expectation terms. Notably, for the second-order Heun's method, our theoretical framework establishes improved results compared to prior works when reducing to diffusion models. 
    
    \item We have implemented both the forward Euler method and the Heun's method to provide extensive numerical validation using both 2D distributions and high-dimensional Gaussian mixtures. These numerical results further validate our theoretical findings on convergence rates, demonstrating the practical implications of our theoretical results.
\end{itemize}

\section{Related Works}

\subsection{Stochastic Interpolants Analysis}

The stochastic interpolant framework originates from continuous-time normalizing flows, providing a principled approach for constructing data-to-data generative models \citep{albergo2023building}. Subsequent work \citep{albergo2023interpolant} extended this framework through the incorporation of Gaussian perturbations, enabling learnable score functions and facilitating the application of stochastic differential equations for data transformation tasks.

Regarding ODE-based formulations, \cite{albergo2023building} established Wasserstein error bounds for velocity field estimation under Lipschitz continuity assumptions. \cite{benton2024error} advanced these results by considering time-dependent Lipschitz constants, thereby obtaining tighter error bounds. Their analysis further examined the control of Lipschitz constants specifically for linear interpolants. Both studies focused exclusively on continuous-time settings.

In the SDE context, \cite{albergo2023interpolant} derived continuous-time Kullback-Leibler (KL) divergence bounds between target and estimated distributions, expressed in terms of the mean squared error of drift estimation. \cite{liu2025finitetime} made progress by establishing finite-time error bounds for the Euler-Maruyama discretization scheme, representing the first discrete-time analysis within the stochastic interpolant framework. Their work additionally investigated the impact of schedule selection on convergence properties.

\subsection{ODE-based Diffusion Models Analysis }

Recent theoretical work has made significant progress in analyzing the convergence properties of probability-flow ODE diffusion models. \cite{chen2023provably} developed a provably efficient sampling algorithm incorporating corrector steps within the ODE solver framework, though this approach introduces additional stochasticity to the process.

Theoretical analyses of purely deterministic generation typically adopt two principal approaches. The first directly examines discrete-time density evolution, where \cite{li2024faster} established a foundational framework for controlling the total variation (TV) distance between target and estimated distributions. Subsequent work \citep{li2024sharp,li2024accelerating} extended this framework to derive tighter bounds and accelerated convergence guarantees. The second approach considers equivalent continuous-time processes through partial differential equation (PDE) analysis. While both \cite{huang2025convergence} and \citep{li2025unified} employed PDE techniques to bound TV error evolution, the latter achieved superior results through more refined error decomposition methods.

Current theoretical understanding indicates that first-order methods require $\tilde{O}(d/\varepsilon)$ iterations to achieve $\varepsilon$-precision in TV distance \citep{li2024sharp}. However, existing analyses of higher-order methods remain either non-tight or reliant on overly restrictive assumptions, highlighting the need for more comprehensive theoretical frameworks for ODE-based approaches.

%% file: contents/background.tex
\section{Background on Stochastic Interpolants}

Consider two probability distributions, $\rho_0$ and $\rho_1$, defined on $\mathbb{R}^d$. The stochastic interpolant framework \citep{albergo2023building,albergo2023interpolant} provides an approach for constructing a learnable mapping between $\rho_0$ and $\rho_1$. Specifically, this is accomplished by identifying a learnable vector field $b(t,x)$ such that the solution $(X_t)_{t\in[0,1]}$ to the ordinary differential equation
$$\dif X_t=b(t,X_t)\dif t$$
with initial condition $X_0\sim\rho_0$, satisfies $X_1\sim\rho_1$. Notably, this ODE formulation is similar to the flow matching framework.

To determine the vector field $b(t,x)$, we introduce the stochastic interpolant between $\rho_0$ and $\rho_1$ as a stochastic process defined by
$$x_t=I(t,x_0,x_1)+\gamma(t)z,$$
where $(x_0,x_1)\sim\nu$ (with $x_0\sim\rho_0$, $x_1\sim\rho_1$), $z\sim\mathcal{N}(0,I_d)$ is an independent Gaussian noise term. Notably, $\nu$ may accommodate both paired and unpaired data settings. The $C^2$-smooth interpolation $I(t,x_0,x_1)$ satisfies the boundary conditions $I(0,x_0,x_1)=x_0$ and $I(1,x_0,x_1)=x_1$. The latent term $\gamma(t)z$ serves to regularize the marginal distributions $\rho(t,\cdot)$ of $(x_t)_{t\in[0,1]}$. Typically, $\gamma(t)$ is chosen such that 
\begin{itemize}
    \item $\gamma(0)=\gamma(1)=0$ (ensuring $\rho(0)=\rho_0$ and $\rho(1)=\rho_1$),
    \item $\gamma(t)>0$ for $t\in(0,1)$ (maintaining smoothness in the intermediate states).
\end{itemize} 
A widely used example is the linear stochastic interpolant $x_t=(1-t)x_0+tx_1+\sqrt{2t(1-t)}z$, where $\gamma(t)=\sqrt{2t(1-t)}$. This choice is particularly notable because, when $x_0\sim\mathcal{N}(0,I_d)$, the same marginal distributions of $x_t$ coincide with those of the variance-preserving diffusion model, establishing a direct connection between interpolant-based methods and diffusion processes.

Although $x_t$ establishes a connection between the two distributions, its computation depends on both $x_0$ and $x_1$, making it unsuitable as a generative model yet. However, \cite{albergo2023interpolant} demonstrated that the marginal density $\rho(t,x)$ satisfies the following transport equation (where the notation $\nabla$ represents the gradient operator):
\begin{equation}
    \partial_t\rho(t,x)+\nabla\cdot(\rho(t,x)b(t,x))=0.
    \label{eq:transport}
\end{equation}
Here, the velocity field $b(t,x)$ is given by
$$b(t,x)=\mathbb{E}[\dot{x}_t|x_t=x]=\mathbb{E}[\partial_tI(t,x_0,x_1)+\dot{\gamma}_tz|x_t=x].$$
This transport equation reveals that any process $(X_t)_{t\in[0,1]}$ satisfying the initial condition $X_0\sim\rho_0$ and solving the ODE
\begin{equation}
    \dif X_t=b(t,X_t)\dif t,
    \label{eq:real-ode}
\end{equation}
will share identical marginal distributions with the stochastic interpolant $(x_t)_{t\in[0,1]}$, i.e., $X_t\sim\rho(t)$ for all $t\in[0,1]$. Crucially, the temporal derivative of $X_t$ depends only on the current time $t$ and position $X_t$. Consequently, when initialized with a sample $x_0\sim\rho_0$, solving \eqref{eq:real-ode} yields $X_1\sim\rho_1$. 

Furthermore, for any non-negative function $\epsilon(t)\ge0$, \cite{albergo2023interpolant} shows that the solution to the SDE
$$\dif X_t=[b(t,X_t)+\epsilon(t)s(t,X_t)]\dif t+\sqrt{2\epsilon(t)}\dif W_t$$
also transforms samples from $\rho_0$ into $\rho_1$, where $W_t$ denotes the standard Wiener process and $s(t,x)=\nabla\log\rho(t,x)=\mathbb{E}[\gamma^{-1}z|x_t=x]$ represents the well-known score function. Notably, the ODE in \eqref{eq:real-ode} emerges as a special case of this SDE when $\epsilon(t)\equiv0$.

In practice, an estimator $\hat{b}(t,x)$ of the expected velocity field $b(t,x)$ can be obtained by minimizing the quadratic loss:
$$\mathcal{L}[\hat{b}]=\int_0^1\mathbb{E}\left[\frac{1}{2}\Vert\hat{b}(t,x_t)\Vert^2-b(t,x_t)\cdot\dot{x}_t\right]\dif t,$$
where $x_t\sim\rho(t)$ denotes the stochastic interpolant process. This loss differs from the mean squared error $$\text{MSE}=\int_0^1\mathbb{E}\left[\frac{1}{2}\Vert\hat{b}(t,x_t)-b(t,x_t)\Vert^2\right]\dif t$$
by a constant that is independent of $\hat{b}$, as seen by expanding the quadratic term. The resulting estimator $\hat{b}(t,x)$, when used in \eqref{eq:real-ode}, yields a generative model that transports $\rho_0$ to $\rho_1$.

%% file: contents/results-1.tex
\section{Main Results for Using the Forward Euler Method}

In this section, we analyze the forward Euler method, a fundamental first-order approximation technique for ordinary differential equations (ODEs). Initially, we provide a formal definition of the method when applied to the stochastic interpolant case. Given a time discretization, a schedule $\{t_k\}_{k=0}^N$ satisfying $t_0<t_1<t_2<\cdots<t_N$ is specified. Let $\hat{X}_{t_0}$ denote the initial condition for the equation. At the $(k+1)$-th iteration, the forward Euler method approximates the solution by $$\hat{X}_{t_{k+1}}=\hat{X}_{t_k}+h_{t_k}\cdot\hat{b}(t_k,\hat{X}_{t_k}),$$
where $\hat{b}(t,x)$ represents an estimator for the true drift function $b(t,x)$, and $h_{t_k}=t_{k+1}-t_k$ is the step size. If $\hat{X}_{t_k}$ denotes the true solution of $\hat{\rho}(t_k)$, our objective is to quantify and control the discrepancy between the approximated terminal distribution $\hat{\rho}(t_N)$ and the true terminal distribution $\rho(t_N)$.

To facilitate a rigorous analysis of the forward Euler method, we introduce the following set of assumptions, which are requisite for our subsequent derivations.

\begingroup
\renewcommand\thetheorem{1}
\begin{assumption}
    $\underset{(x_0,x_1)\sim\nu}{\mathbb{E}}\left[\Vert x_0-x_1\Vert^4\right]<\infty$. Furthermore, there exist positive constants $C_I,C_\gamma>0$ such that for all $x_0,x_1\in\mathbb{R}^d$ and $p\in\{1,2\}$,
    $$\begin{aligned}
        \Vert\partial_t^pI(t,x_0,x_1)\Vert\le C_I\Vert x_0-x_1\Vert,\\
        \left|\frac{\dif^p}{\dif t^p}\left[\gamma^2(t)\right]\right|\le C_\gamma.
    \end{aligned}$$
    \label{assumption:regularity}
\end{assumption}
\endgroup

Assumption \ref{assumption:regularity} is essential for guaranteeing the requisite regularity of the stochastic interpolant process $(x_t)_{t\in[0,1]}$ and the function $b(t,x)$. In the case of a linear interpolant $x_t=(1-t)x_0+tx_1+\sqrt{2t(1-t)}z$, this assumption is satisfied if both $\rho_0$ and $\rho_1$ possess finite fourth moments \citep{albergo2023interpolant}.

\begingroup
\renewcommand\thetheorem{2}
\begin{assumption}
    The estimator satisfies
    $$\sum_{k=0}^{N-1}h_k\underset{x_{t_k}\sim\rho(t_k)}{\mathbb{E}}\left[\varepsilon_{1,k}(x_{t_k})^2\right]\le\varepsilon_{\text{drift}}^2<\infty,$$
    where $\varepsilon_{1,k}(x)=\Vert\hat{b}(t_k,x)-b(t_k,x)\Vert$.
    \label{assumption:drift-error}
\end{assumption}
\endgroup

Assumption \ref{assumption:drift-error} is a standard condition concerning the quality of the estimator, aligning with similar assumptions found in prior works \citep{benton2024nearly,liu2025finitetime}. Unlike continuous-time mean squared error conditions, this assumption provides a discrete-time formulation, focusing on the error specifically at the time steps utilized by the numerical method.

However, in contrast to models based solely on stochastic differential equations (e.g., \citealt{liu2025finitetime}), Assumption \ref{assumption:drift-error} alone is insufficient to control the overall distribution error \citep{li2024faster,li2024accelerating,li2025unified}. Consequently, we introduce further assumptions regarding the properties of $\hat{b}(t,x)$.

\begingroup
\renewcommand\thetheorem{3}
\begin{assumption}
    The estimator satisfies
    $$\sum_{k=0}^{N-1}h_k\underset{x_{t_k}\sim\rho(t_k)}{\mathbb{E}}\left[\varepsilon_{2,k}(x_{t_k})\right]\le\varepsilon_{\text{div}}<\infty,$$
    where $\varepsilon_{2,k}(x)=\left\Vert\nabla\hat{b}(t_k,x)-\nabla b(t_k,x)\right\Vert_F$. Here the notation $\nabla^p b(t,x)$ represents the tensor consisting of the $p$-th order derivatives of $b$ with respect to $x$ (e.g., when $p=1$, this notation represents the Jacobian matrix).
    \label{assumption:div-error}
\end{assumption}
\endgroup

\begingroup
\renewcommand\thetheorem{4}
\begin{assumption}
    $\hat{b}(t,x)$ is $C^2$ w.r.t. $x$. Furthermore, both $\hat{b}(t,x)$ and $\nabla\cdot\hat{b}(t,x)$ are Lipschitz continuous w.r.t. $x$. Specifically, there exists a constant $L>0$ such that for all $k=0,1,\dots,N$ and all $x\in\mathbb{R}^d$,
    $$\left\Vert\nabla\hat{b}(t_k,x)\right\Vert_F\le L,\quad\left\Vert\nabla^2\hat{b}(t_k,x)\right\Vert_F\le L^{3/2}.$$
    In addition, for all $x,y\in\mathbb{R}^d$, 
    $$\Vert\nabla\cdot\hat{b}(t,x)-\nabla\cdot\hat{b}(t,y)\Vert\le L^{3/2}\Vert x-y\Vert.$$
    \label{assumption:lipschitz}
\end{assumption}
\endgroup

Assumption \ref{assumption:div-error} extends the requirements to the Jacibian matrices of $\hat{b}(t,x)$, stipulating its proximity to the true divergence. Assumption \ref{assumption:lipschitz} imposes a uniform Lipschitz constant on both $\hat{b}(t,x)$ and its divergence. The reasonableness of Assumption \ref{assumption:lipschitz} can be illustrated by considering cases where the data for both $\rho_0$ and $\rho_1$ are bounded in each dimension; in such scenarios, Lemma \ref{lem:bound-uniform} in the Appendix demonstrates that $L$ is of order $O(d)$.

With the preceding assumptions established, we are now ready to present the main theoretical result concerning the application of the forward Euler method.

\begingroup
\renewcommand\thetheorem{5}
\begin{theorem}
    Under Assumptions \ref{assumption:regularity}, \ref{assumption:drift-error}, \ref{assumption:div-error} and \ref{assumption:lipschitz}, suppose the forward Euler method is initialized with $\hat{X}_{t_0}\sim\hat{\rho}(t_0)$, and the step sizes satisfy $h_k\le\frac{1}{2L}$. Then,
    $$\begin{aligned}
        \textnormal{TV}(\rho(t_N),\hat{\rho}(t_N))&\lesssim\textnormal{TV}(\rho(t_0),\hat{\rho}(t_0))+\varepsilon_{\textnormal{div}}+\varepsilon_{\textnormal{drift}}\left(d^{1/2}S(\gamma,t_0,t_N)^{1/2}+L^{1/2}\right)\\
        &\quad+\underbrace{\sum_{k=0}^{N-1}h_k^2\left[\bar{\gamma}_k^{-4}d^2+\bar{\gamma}_k^{-2}M^2\right].}_{\textnormal{Discretization Error}}
    \end{aligned}$$
    Here, the terms are defined as:
    $$\begin{cases}
        \bar{\gamma}_k&:=\inf_{t\in[t_k,t_{k+1}]}\gamma(t),\\
        S(\gamma,t_0,t_N)&:=\int_{t_0}^{t_N}\gamma^{-2}(t)\dif t,\\
        M&:=\max\left\{d,L,\sqrt{\mathbb{E}_\nu\left[\Vert x_0-x_1\Vert^4\right]}\right\}.
    \end{cases}$$
    \label{thm:euler}
\end{theorem}
\endgroup

Theorem \ref{thm:euler} provides an comprehensive upper bound for the total variation (TV) distance between the true target distribution $\rho(t_N)$ and the distribution approximated by the forward Euler method, $\hat{\rho}(t_N)$. This bound elucidates the influence of several critical factors: the initialization error $\textnormal{TV}(\rho(t_0),\hat{\rho}(t_N))$, dimension $d$, the distance between source and target distribution captured by $\mathbb{E}_\nu[\Vert x_0-x_1\Vert^4]$, the Lipschitz constant $L$, the latent scale term $\gamma(t)$, the estimation errors $\varepsilon_{\textnormal{drift}}$ and $\varepsilon_{\textnormal{div}}$, and finally, the step sizes $\{h_k\}_{k=0}^{N-1}$. 

To avoid an unbounded right-hand side in Theorem \ref{thm:euler} when $\gamma(0)=0$ or $\gamma(1)=0$ (which ensures $\rho(0)=\rho_0$ and $\rho(1)=\rho_1$ in the stochastic interpolant definition), we follow \cite{liu2025finitetime} and simulate the process within a sub-interval $[t_0,t_N]\subset(0,1)$. This means our sampling starts from an estimation of $\rho(t_0)$ rather than $\rho_0$, and similarly aims for an estimation of $\rho(t_N)$ instead of $\rho_1$. This approach is justified because $\rho(t_0)$ and $\rho(t_N)$ are close to $\rho_0$ and $\rho_1$, respectively, when $t$ is close to $0$ and $t_N$ is close to $1$. In addition, the initialization error can be made very small when $\rho_0$ is available (e.g., if $I(t_0,x_0,x_1)=x_0$). This technique is also known as early stopping in the context of diffusion models \citep{song2021scorebased,chen2023improved,benton2024nearly}.

The step sizes $\{h_k\}_{k=0}^{N-1}$ specifically influence the last term of Theorem \ref{thm:euler}, which directly corresponds to the convergence rate of the forward Euler method. Since this error term is proportional to $h_k^2$, it approaches zero as the step sizes tend to zero. The remaining terms in the theorem are unaffected by the choice of step sizes, representing the inherent quality achievable with an infinite number of steps.

\subsection{Complexity of Forward Euler Method}

Theorem \ref{thm:euler} broadly applies to various definitions of $I$, $\gamma$ and step size schedules $\{h_k\}_{k=0}^{N-1}$, provided the stochastic interpolant assumptions are met. This section explores two specific stochastic interpolant definitions and their corresponding computational complexities.

\paragraph{Stochastic Interpolants with $\gamma(t)=\sqrt{at(1-t)}$} This choice, previously examined by \cite{liu2025finitetime}, is natural as $\gamma^2(t)$ represents the variance of a Brownian bridge. It satisfies $\gamma(0)=\gamma(1)=0$ and Assumption \ref{assumption:regularity}.

To optimize the schedule, we focus on minimizing the last term in Theorem \ref{thm:euler}: $$\varepsilon\lesssim\sum_{k=0}^{N-1}h_k^2\left[\bar{\gamma}_k^{-4}d^2+\bar{\gamma}_k^{-2}M^2\right].$$
Following \cite{liu2025finitetime}, the error's proportionality to $h_k^2\bar{\gamma}_k^{-4}$ suggests setting $h_k\propto\bar{\gamma}_k^2$. This leads to the following schedule: define a midpoint $m$ at $t_m=0.5$ and select a step size scale parameter $h>0$. The schedule $\{t_k\}_{k=0}^N$ is then given by
\begin{equation}
    \begin{cases}
        t_k=\frac{1}{2}(1-h)^{m-k},&k\le m;\\
        t_k=1-\frac{1}{2}(1-h)^{k-m},&k>m.
    \end{cases}
    \label{eq:schedule2}
\end{equation}
This schedule ensures $h_k\bar{\gamma}_k^{-2}=O(h)$, resulting in a total number of steps $N=\Theta\left(\frac{1}{h}\log\left(\frac{1}{t_0(1-t_N)}\right)\right)$.

Building on the schedule defined above, and setting $\delta=\min\{t_0,1-t_N\}$, Theorem \ref{thm:euler} yields the following error bound:
$$\begin{aligned}
    \textnormal{TV}(\rho(t_N),\hat{\rho}(t_N))&\lesssim\textnormal{TV}(\rho(t_0),\hat{\rho}(t_0))+\varepsilon_{\textnormal{div}}+\varepsilon_{\textnormal{drift}}\left(d^{1/2}\log^{1/2}(1/\delta)+L^{1/2}\right)\\
    &\quad+\underbrace{h\left[d^2\log(1/\delta)+M^2\right]}_{\textnormal{Discretization Error}}.
\end{aligned}$$
This indicates that achieving an $\varepsilon$ TV error requires $N=O\left(\frac{1}{\varepsilon}\left[d^2\log^2(1/\delta)+M^2\log(1/\delta)\right]\right)$ steps.

\paragraph{Variance-Preserving Diffusion Models}

Now, consider a stochastic interpolant defined by $x_1\sim\rho_{\text{data}}$, $z\sim\mathcal{N}(0,I_d)$ and $x_t=tx_1+\sqrt{1-t^2}z$. Here, $\rho(0)=\mathcal{N}(0,I_d)$ and $\rho(1)=\rho_{\text{data}}$, and the marginal distributions align with those of variance-preserving diffusion models. 

Similar to the previous case, we aim for $h_k\propto\gamma^2=1-t^2=\Theta(1-t)$. This leads to the schedule:
\begin{equation}
    t_k=1-(1-h)^k,\quad 0\le k\le N,
    \label{eq:schedule1}
\end{equation}
where $h>0$ controls the step sizes.

If $\delta=1-t_N$ is the early-stopping time, the number of steps $N$ is $N=\Theta\left(\frac{1}{h}\log\frac{1}{\delta}\right)$. Assuming $\mathbb{E}[\Vert x_1\Vert^4]=O(d^2)$ and $L=O(d)$ (which holds, e.g., if $\rho_{\text{data}}$ is bounded in each dimension), by the previous discussion, achieving an $\varepsilon$ error requires $N=O\left(\frac{1}{\varepsilon}d^2\log^2\frac{1}{\delta}\right)$ steps. This $O(\varepsilon^{-1})$ iteration complexity matches prior theoretical findings \citep{li2024faster,li2025unified,li2024sharp}. While \cite{li2024sharp} reports a faster $\tilde{O}(d/\varepsilon)$ complexity, their result relies on stochastic localization, which is restricted to Gaussian-to-data scenarios. Our analysis, however, extends to the more general data-to-data case.

\subsection{Proof Sketch of Theorem \ref{thm:euler}}

We'll briefly outline the proof for Theorem \ref{thm:euler}; the complete details are in the Appendix. The proof hinges on Lemma \ref{lem:tv-dist} (Lemma 3.2 by \citealt{li2025unified}), which provides a way to control the TV distance between two processes.

\begingroup
\renewcommand\thetheorem{6}
\begin{lemma} (Lemma 3.2 by \citealt{li2025unified})
    For two processes $X_t$ and $\hat{X}_t$ governed by $\dif X_t=b(t,X_t)\dif t$ and $\dif\hat{X}_t=\hat{b}(t,\hat{X}_t)\dif t$, respectively. Let $\rho(t)$ and $\hat{\rho(t)}$ be their corresponding laws. Then:
    $$\begin{aligned}
        \frac{\dif}{\dif t}\textnormal{TV}(\rho(t),\hat{\rho}(t))&=\int_{\Omega_t}(\nabla\cdot b(t,x)-\nabla\cdot\hat{b}(t,x))\rho(t,x)\dif x\\
        &-\int_{\Omega_t}(b(t,x)-\hat{b}(t,x))\nabla\log\rho(t,x)\rho(t,x)\dif x,
    \end{aligned}$$
    where $\Omega_t=\{x\in\mathbb{R}^d:\hat{\rho}(t,x)>\rho(t,x)\}$.
    \label{lem:tv-dist}
\end{lemma}
\endgroup

To apply Lemma \ref{lem:tv-dist}, we must express the discrete-time forward Euler process $\{\hat{X}_{t_k}\}_{k=0}^{N}$ as a continuous-time process. This is achieved by the following interpolation: 
$$\hat{X}_t=F_{t_k\to t}(\hat{X}_{t_k})=\hat{X}_{t_k}+(t-t_k)\hat{b}(t_k,\hat{X}_{t_k}).$$
Provided Assumption \ref{assumption:lipschitz} holds and the step size $h_k$ is sufficiently small, $F_{t_k\to t}$ becomes a diffeomorphism from $\mathbb{R}^d$ to itself. This allows us to write $$\dif\hat{X}_t=\hat{b}(t_k,\hat{X}_{t_k})\dif t=\hat{b}(t_k,F_{t_k\to t}^{-1}(\hat{X}_t))\dif t.$$
By defining $\tilde{b}(t,x)=\hat{b}(t_k,F_{t_k\to t}^{-1}(x))$, the equation takes the form $\dif\hat{X}_t=\tilde{b}(t,\hat{X}_t)$, which is precisely the structure required by Lemma \ref{lem:tv-dist}.

The remainder of the proof involves demonstrating the proximity of $\tilde{b}(t,x)$ to $b(t,x)$, and similarly for their divergences. While similar in spirit to \cite{li2025unified}, our approach partitions the error differently. Let $z=F_{t_k\to t}^{-1}(X_t)$. The core idea is a sequential approximation:

For the drift term:
$$\begin{aligned}
    \tilde{b}(t,X_t)&=\hat{b}(t_k,z)\overset{\text{(a)}}{\approx}\hat{b}(t_k,X_{t_k})\overset{\text{(b)}}{\approx} b(t_k,X_{t_k})\overset{\text{(c)}}{\approx} b(t,X_t),
\end{aligned}$$
For the divergence term:
$$\begin{aligned}
    \nabla\cdot\tilde{b}(t,X_t)
    &=\text{tr}\left[\nabla\hat{b}(t_k,z)\cdot\nabla F_{t_k\to t}(z)^{-1}\right]\overset{\text{(a)}}{\approx}\text{tr}\left[\nabla\hat{b}(t_k,X_{t_k})\cdot I_d\right]\\
    &\overset{\text{(b)}}{\approx}\nabla\cdot b(t_k,X_{t_k})\overset{\text{(c)}}{\approx}\nabla\cdot b(t,X_t).
\end{aligned}$$
Here, step (a) relies on $z\approx X_{t_k}$ and $\nabla F_{t_k\to t}(z)\approx I_d$. Step (b) accounts for estimation errors, leveraging Assumptions \ref{assumption:drift-error} and \ref{assumption:div-error}. Finally, step (c) quantifies the discretization error between time $t_k$ and $t$.

%% file: contents/results-2.tex
\section{Main Results for Using Heun's Method}

This section delves into Heun's method, a widely recognized second-order approximation for ordinary differential equations (ODEs). Formally, given a sequence of discrete time steps $\{t_k\}_{k=0}^{N}$, let $\hat{X}_{t_k}$ denote the estimated solution at time $t_k$. Beginning with an initial condition $\hat{X}_{t_0}$, the method proceeds via the following iterative scheme:
$$\begin{cases}
    \tilde{X}_{t_{k+1}}=\hat{X}_{t_k}+h_{t_k}\cdot\hat{b}(t_k,\hat{X}_{t_k}),\\
    \hat{X}_{t_{k+1}}=\hat{X}_{t_k}+\frac{1}{2}h_{t_k}\left[\hat{b}(t_k,\hat{X}_{t_k})+\hat{b}(t_{k+1},\tilde{X}_{t_{k+1}})\right].
\end{cases}$$
As with the first-order case, our primary objective is to quantify and control the difference between the estimated distribution $\hat{\rho}(t_N)$ and the true distribution $\rho(t_N)$, where $\hat{\rho}(t_k)$ represents the distribution of $\hat{X}_{t_k}$. To accommodate this higher-order numerical method, certain assumptions require modification.

\begingroup
\renewcommand\thetheorem{1'}
\begin{assumption}
    $\underset{(x_0,x_1)\sim\nu}{\mathbb{E}}\left[\Vert x_0-x_1\Vert^6\right]<\infty$. Furthermore, there exist positive constants $C_I,C_\gamma$ such that for all $x_0,x_1\in\mathbb{R}^d$ and $p\in\{1,2\}$,
    $$\begin{aligned}
        \Vert\partial_t^pI(t,x_0,x_1)\Vert\le C_I\Vert x_0-x_1\Vert,\\
        \left|\frac{\dif^p}{\dif t^p}\left[\gamma^2(t)\right]\right|\le C_\gamma.
    \end{aligned}$$
    \label{assumption:regularity-2}
\end{assumption}
\endgroup

\begingroup
\renewcommand\thetheorem{2'}
\begin{assumption}
    The estimator fulfills the condition
    $$\sum_{k=0}^{N-1}h_k\mathbb{E}\left[\varepsilon_{1,k}(x_{t_k})^2+\varepsilon_{1,k+1}(x_{t_{k+1}})^2\right]\le\varepsilon_{\text{drift}}^2<\infty,$$
    where $\varepsilon_{1,k}(x)=\Vert\hat{b}(t_k,x)-b(t_k,x)\Vert$.
    \label{assumption:drift-error-2}
\end{assumption}
\endgroup

\begingroup
\renewcommand\thetheorem{3'}
\begin{assumption}
    The estimator satisfies
    $$\sum_{k=0}^{N-1}h_k\mathbb{E}\left[\varepsilon_{2,k}(x_{t_k})^2+\varepsilon_{2,k+1}(x_{t_{k+1}})^2\right]^{1/2}\le\varepsilon_{\text{div}}<\infty,$$
    where $\varepsilon_{2,k}(x)=\left\Vert\nabla\hat{b}(t_k,x)-\nabla b(t_k,x)\right\Vert_F$.
    \label{assumption:div-error-2}
\end{assumption}
\endgroup

Assumptions \ref{assumption:regularity-2}, \ref{assumption:drift-error-2}, and \ref{assumption:div-error-2} represent refined versions of Assumptions \ref{assumption:regularity}, \ref{assumption:drift-error}, and \ref{assumption:div-error}, respectively, specifically adapted for Heun's method. Assumption \ref{assumption:regularity-2} mandates a finite sixth moment (in contrast to a fourth moment) and requires $I$ and $\gamma$ to possess bounded derivatives up to the third order (rather than the second). Assumptions \ref{assumption:drift-error-2} and \ref{assumption:div-error-2} incorporate an additional term, reflecting the fact that Heun's method necessitates two evaluations of $\hat{b}(t,x)$ per step. It's worth noting that Assumption \ref{assumption:div-error-2} utilizes $\varepsilon_{2,k}^2$ instead of $\varepsilon_{2,k}$; this is equivalent to employing $\varepsilon_{2,k}$ when a uniform bound on $\varepsilon_{2,k}$ exists, which occurs when both $\nabla\cdot\hat{b}(t,x)$ and $\nabla\cdot b(t,x)$ are bounded.

With the preceding assumptions established, we are now prepared to present our main theorem for Heun's method.

\begingroup
\renewcommand\thetheorem{7}
\begin{theorem}
    Under Assumptions \ref{assumption:regularity-2},\ref{assumption:drift-error-2},\ref{assumption:div-error-2},\ref{assumption:lipschitz}, if the equation is solved using Heun's method with an initial condition $\hat{X}_{t_0}\sim\hat{\rho}(t_0)$, and provided the step sizes satisfy $h_k\le\min\left\{\frac{1}{4L},\mathbb{E}[\Vert x_0-x_1\Vert^6]^{-1/3},d^{-1}\bar{\gamma}_k^2\right\}$, then the TV distance is bounded as follows:
    $$\begin{aligned}
        \textnormal{TV}(\rho(t_N),\hat{\rho}(t_N))&\lesssim\textnormal{TV}(\rho(t_0),\hat{\rho}(t_0))+\varepsilon_{\textnormal{div}}+\varepsilon_{\textnormal{drift}}\left(d^{1/2}S(\gamma,t_0,t_N)^{1/2}+L^{1/2}\right)\\
        &+\underbrace{\sum_{k=0}^{N-1}h_k^{3}\left[\bar{\gamma}_k^{-6}d^3+\bar{\gamma}_k^{-4}M^3\right].}_{\textnormal{Discretization Error}}
    \end{aligned}$$
    Here, the terms are defined as
    $$\begin{cases}
        \bar{\gamma}_k&:=\inf_{t\in[t_k,t_{k+1}]}\gamma(t),\\
        S(\gamma,t_0,t_N)&:=\int_{t_0}^{t_N}\gamma^{-2}(t)\dif t,\\
        M&:=\max\left\{d,L,\mathbb{E}_\nu\left[\Vert x_0-x_1\Vert^6\right]^{1/3}\right\},
    \end{cases}$$
    \label{thm:heun}
\end{theorem}
\endgroup

Similar to Theorem \ref{thm:euler}, Theorem \ref{thm:heun} establishes an upper bound for the TV distance between the estimated distribution and the true target distribution. The primary distinction lies in the last term, where the order of the step sizes is enhanced from $h_k^2$ to $h_k^3$, reflecting the higher-order accuracy of Heun's method.

\subsection{Complexity of Heun's Method}

Similar to the forward Euler case, we analyze the complexity of Heun's method for the same stochastic interpolant instances.

\paragraph{Stochastic Interpolants with $\gamma(t)=\sqrt{at(1-t)}$}

For this case, the last term in Theorem \ref{thm:heun} is proportional to $h_k^3\bar{\gamma}_k^{-6}$. Balancing the error again suggests $h_k\propto\bar{\gamma}_k^2$ leading to the exponentially decaying schedule \eqref{eq:schedule2}. Let $\delta=\min\{t_0,1-t_N\}$. Theorem \ref{thm:heun} then yields
$$\begin{aligned}
    \textnormal{TV}(\rho(t_N),\hat{\rho}(t_N))&\lesssim\textnormal{TV}(\rho(t_0),\hat{\rho}(t_0))+\varepsilon_{\textnormal{div}}+\varepsilon_{\textnormal{drift}}\left(d^{1/2}\log^{1/2}(1/\delta)+L^{1/2}\right)\\
    &\quad+\underbrace{h^2\left[d^3\log(1/\delta)+M^3\right]}_{\textnormal{Discretization Error}}
\end{aligned}$$
Since $N=\Theta(h^{-1}\log(1/\delta))$, the above inequality shows that achieving an $\varepsilon$-TV error requires $$N=O\left(\frac{1}{\sqrt{\varepsilon}}\left[d^{3/2}\log^{3/2}(1/\delta)+M^{3/2}\log(1/\delta)\right]\right)$$ steps. This represents a reduction in complexity from $O(1/\varepsilon)$ (for forward Euler) to $O(1/\sqrt{\varepsilon})$.

\paragraph{Variance-Preserving Diffusion Models} For $x_t=tx_1+\sqrt{1-t^2}z$, using schedule \eqref{eq:schedule1}, we derive a similar complexity of $N=O\left(\frac{1}{\sqrt{\varepsilon}}\left[d^{3/2}\log^{3/2}(1/\delta)+M^{3/2}\log(1/\delta)\right]\right)$,
where $\delta=1-t_N$ is the early-stopping time. When $M=O(d)$, this simplifies to
$N=O\left(\frac{1}{\sqrt{\varepsilon}}d^{3/2}\log^{3/2}(1/\delta)\right)$. This $O(1/\sqrt{\varepsilon})$ convergence rate aligns with previous works \citep{li2024accelerating, huang2025convergence}. Compared to the $\tilde{O}(\varepsilon^{-1/2}d^3)$ complexity by \cite{li2024accelerating} (for a second-order method assuming bounded data support), our dependence on dimension $d$ is lower. This improvement stems from an additional Lipschitzness assumption that holds in many practical scenarios. While \cite{huang2025convergence} achieve a comparable $\tilde{O}(\varepsilon^{-1/2}d^{3/2})$ complexity, their analysis relies on stricter assumptions. Specifically, they require uniform boundedness of derivatives up to the third order for the score estimator $\hat{s}(t,x)$ w.r.t. both $t$ and $x$, even though the $t$-derivative for the true score function can be unbounded. In contrast, we only necessitate bounded second-order derivatives of $\hat{b}(t,x)$ w.r.t $x$. Furthermore, their Assumption 2 (analogous to our Assumption \ref{assumption:drift-error-2}) considers errors across the entire interval $t\in[t_0,t_N]$, whereas our assumption only requires consideration at the discrete time steps $\{t_k\}_{k=0}^N$ actually used by Heun's method.

\subsection{Proof Sketch of Theorem \ref{thm:heun}}

The proof of Theorem \ref{thm:heun}, fully detailed in the Appendix, mirrors the first-order case by using Lemma \ref{lem:tv-dist} to control the TV distance. This requires converting the discrete-time process into a continuous one. For Heun's method, we define the interpolation as:
$$\begin{aligned}
    \hat{X}_t&=G_{t_k\to t}(\hat{X}_{t_k})\\
    &=\hat{X}_{t_k}+\left[(t-t_k)-\frac{(t-t_k)^2}{2(t_{k+1}-t_k)}\right]\hat{b}(t_k,\hat{X}_{t_k})+\frac{(t-t_k)^2}{2(t_{k+1}-t_k)}\hat{b}(t_{k+1},F_{t_k\to t_{k+1}}(\hat{X}_{t_k})).
\end{aligned}$$
This interpolation implies that $\frac{\dif}{\dif t}\hat{X}_t$ linearly interpolates between the estimated derivatives at $t_k$ and $t_{k+1}$:
$\frac{\dif}{\dif t}\hat{X}_t=\frac{t_{k+1}-t}{t_{k+1}-t_k}\hat{b}(t_k,\hat{X}_{t_k})+\frac{t-t_k}{t_{k+1}-t_k}\hat{b}(t_{k+1},F_{t_k\to t_{k+1}}(\hat{X}_{t_k}))$.
This provides a higher-order approximation than the forward Euler interpolation. While other interpolations exist (e.g., the approach of \citealt{huang2025convergence}), our chosen method simplifies analysis by only requiring access to $\hat{b}(t,x)$ at the discrete time steps $t\in\{t_k\}_{k=0}^N$ used in the Heun's iteration.

When the step size is sufficiently small, $G_{t_k\to t}$ is an diffeomorphism, ensuring $\hat{X}_{t_k}=G_{t_k\to t}^{-1}(\hat{X}_t)$ is uniquely determined. Write
$\frac{\dif}{\dif t}\hat{X}_t=\tilde{b}(t,\hat{X}_t):=\partial_tG_{t_k\to t}(G_{t_k\to t}^{-1}(\hat{X}_t))$.
This allows application of Lemma \ref{lem:tv-dist}, requiring us to show that $b(t,x)\approx\tilde{b}(t,x)$ and $\nabla\cdot b(t,x)\approx\nabla\cdot\tilde{b}(t,x)$. Let $z=G_{t_k\to t}^{-1}(X_t)$, where $X_t$ is the true solution of the equation \eqref{eq:real-ode}. We then apply a sequential approximation:

For the drift term:
$$\begin{aligned}
    \tilde{b}(t,X_t)=\partial_tG_{t_k\to t}(z)&\overset{\text{(a)}}{\approx}\frac{t_{k+1}-t}{t_{k+1}-t_k}\hat{b}(t_k,X_{t_k})+\frac{t-t_k}{t_{k+1}-t_k}\hat{b}(t_{k+1},X_{t_{k+1}})\\
    &\overset{\text{(b)}}{\approx}\frac{t_{k+1}-t}{t_{k+1}-t_k}b(t_k,X_{t_k})+\frac{t-t_k}{t_{k+1}-t_k}b(t_{k+1},X_{t_{k+1}})\\
    &\overset{\text{(c)}}{\approx}b(t,X_t).
\end{aligned}$$
For the divergence term:
$$\begin{aligned}
    \nabla\cdot\tilde{b}(t,X_t)&=\text{tr}\bigg[\frac{t_{k+1}-t}{t_{k+1}-t_k}\nabla\hat{b}(t_k,z)\nabla G_{t_k\to t}(z)^{-1}\\
    &\quad+\frac{t-t_k}{t_{k+1}-t_k}\nabla\hat{b}(t_{k+1},F_{t_k\to t_{k+1}}(z))\cdot\nabla F_{t_k\to t_{k+1}}(z)\nabla G_{t_k\to t}(z)^{-1}\bigg]\\
    &\overset{\text{(a)}}{\approx}\text{tr}\bigg[\frac{t_{k+1}-t}{t_{k+1}-t_k}\nabla\hat{b}(t_k,X_{t_k})\nabla G_{t_k\to t}(z)^{-1}\\
    &\quad+\frac{t-t_k}{t_{k+1}-t_k}\nabla\hat{b}(t_{k+1},X_{t_{k+1}})\cdot\nabla G_{t_k\to t_{k+1}}(z)\nabla G_{t_k\to t}(z)^{-1}\bigg]
\end{aligned}$$
$$\begin{aligned}
    &\overset{\text{(b)}}{\approx}\text{tr}\bigg[\frac{t_{k+1}-t}{t_{k+1}-t_k}\nabla b(t_k,X_{t_k})\nabla G_{t_k\to t}(z)^{-1}\\
    &\quad+\frac{t-t_k}{t_{k+1}-t_k}\nabla b(t_{k+1},X_{t_{k+1}})\cdot\nabla G_{t_k\to t_{k+1}}(z)\nabla G_{t_k\to t}(z)^{-1}\bigg]\\
    &\overset{\text{(c)}}{\approx}\text{tr}\left[\nabla b(t,X_t)\right]=\nabla\cdot b(t,X_t).
\end{aligned}$$
Step (a) uses $z\approx X_{t_k}$ and $F_{t_k\to t_{k+1}}(z)\approx G_{t_k\to t_{k+1}}(z)\approx X_{t_{k+1}}$. Step (b) relies on Assumptions \ref{assumption:drift-error-2} and \ref{assumption:div-error-2}. Finally, step (c) accounts for the discretization error. For step (c), the key is that for a function $f(t)$ the interpolation $f(t)\approx\frac{t_{k+1}-t}{t_{k+1}-t_k}f(t_{k})+\frac{t-t_k}{t_{k+1}-t_k}f(t_{k+1})$ has an $O(h_k^2)$ error. Here, we apply this to $f(t)=b(t,X_t)$ and $f(t)=\nabla b(t,X_t)\nabla G_{t_k\to t}(z)$, respectively. For the divergence term, unlike the first-order case where $\nabla F_{t_k\to t}(z)^{-1}\approx I_d$ (due to its $O(h_k)$ error), we group $\nabla G_{t_k\to t}(z)$ and $\nabla b(t,X_t)$ for a combined second-order approximation.

%% file: contents/experiments.tex
\begin{figure}[tb]
    \centering
    \begin{subfigure}[b]{0.30\linewidth}
        \includegraphics[width=\linewidth]{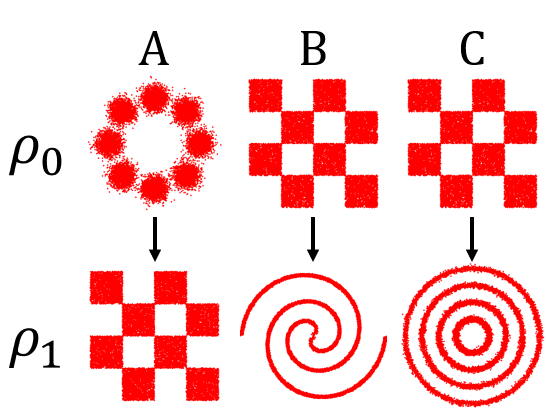}
        \caption{2D density transforms, labeled as task A/B/C, respectively.}
        \label{fig:2dtasks}
    \end{subfigure}
    \begin{subfigure}[b]{0.30\linewidth}
        \includegraphics[width=\linewidth]{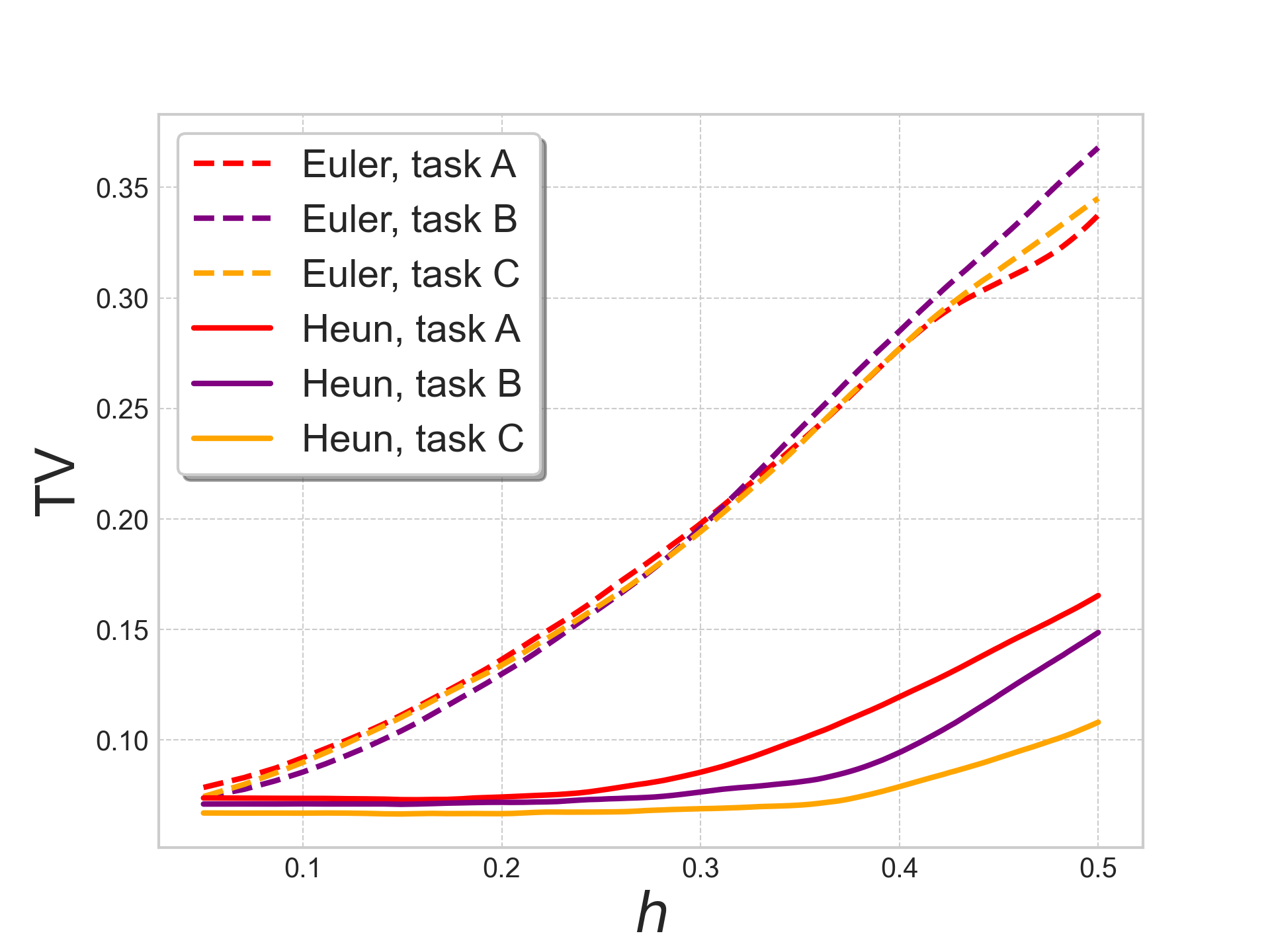}
        \caption{TV error v.s. $h$.}
        \label{fig:tv-h}
    \end{subfigure}
    \begin{subfigure}[b]{0.30\linewidth}
        \includegraphics[width=\linewidth]{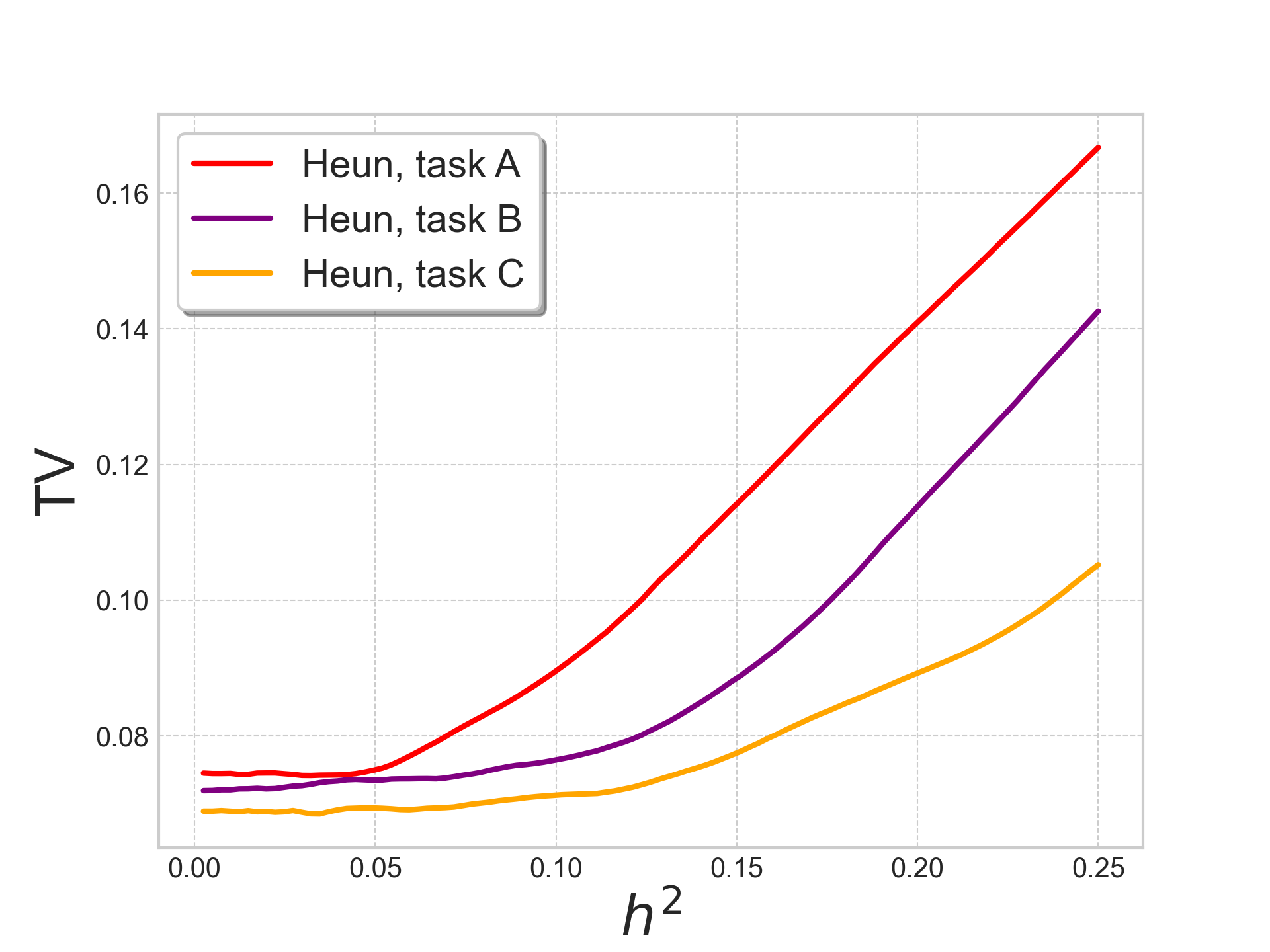}
        \caption{TV error v.s. $h^2$.}
        \label{fig:tv-h2}
    \end{subfigure}
    \caption{Empirical verification of convergence rates for numerical methods. The forward Euler method exhibits $O(h)$ error scaling, while Heun's method demonstrates $O(h^2)$ convergence, consistent with our theoretical predictions.}
    \label{fig:tv-2d}
\end{figure}

\section{Numerical Experiments}

This section presents experimental validation of our theories. Using the interpolant $x_t=(1-t)x_0+tx_1+\sqrt{2at(1-t)}z$ and schedule \eqref{eq:schedule2}, we evaluate both forward Euler and Heun's methods on 2D datasets and $d$-dimensional Gaussian mixtures. 
Our experiments characterize the TV error growth w.r.t.: (i) step size scale $h$ and (ii) dimension $d$.

We evaluate our framework on three 2D dataset pairs from \cite{grathwohl2018scalable} (Figure \ref{fig:2dtasks}). We have designed three different generation tasks: task A transform a mixture of 8 gaussian densities into a checkerboard; task B and C transform the checkerboard into a spiral shape and a mixture of four circles, respectively. Using a neural network to estimate $b(t,x)$, we compare the forward Euler and Heun's methods initialized at $\rho(t_0)$. Figures \ref{fig:tv-h} and \ref{fig:tv-h2} show the empirical TV distances between $\rho(t_N)$ and $\hat{\rho}(t_N)$, confirming our theoretical complexity analysis: the discretization error bounds are $O(h)$ (Euler) and $O(h^2)$ (Heun).

\begin{figure}[tb]
    \centering
    \begin{subfigure}[b]{0.30\linewidth}
        \includegraphics[width=\linewidth]{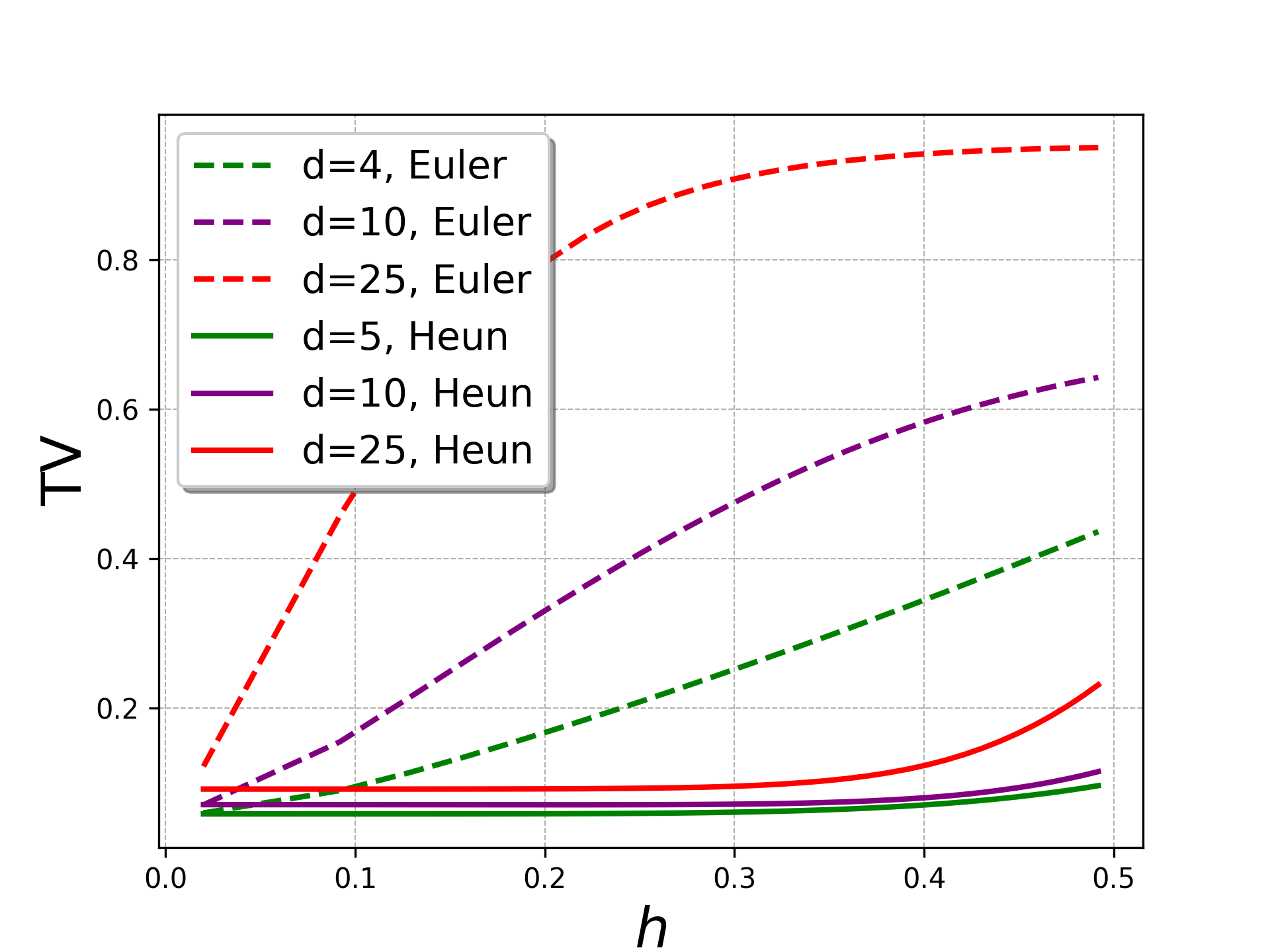}
        \caption{Empirical TV error v.s. $h$.}
        \label{fig:ddvsh}
    \end{subfigure}
    \begin{subfigure}[b]{0.30\linewidth}
        \includegraphics[width=\linewidth]{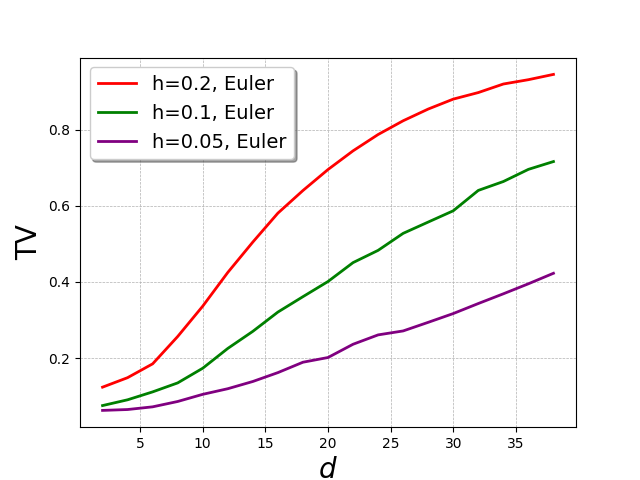}
        \caption{Error growth w.r.t. $d$ (Euler).}
        \label{fig:ddvsd1}
    \end{subfigure}
    \begin{subfigure}[b]{0.30\linewidth}
        \includegraphics[width=\linewidth]{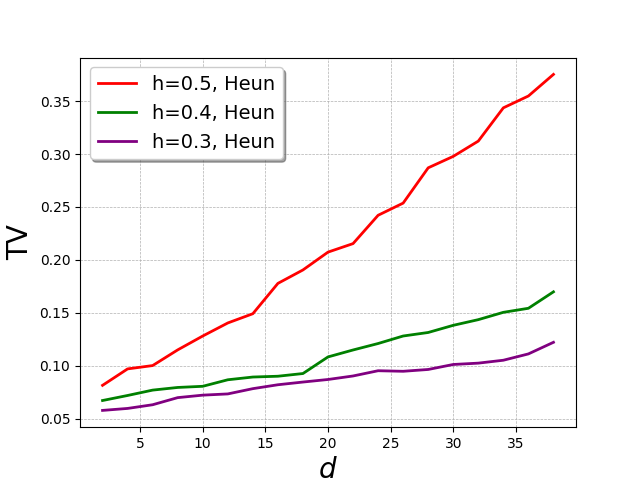}
        \caption{Error growth w.r.t. $d$ (Heun).}
        \label{fig:ddvsd2}
    \end{subfigure}
    \caption{Empirical TV error for $d$-dimensional Gaussian mixtures. (a) confirms the theoretical $h$-dependence, while (b)-(c) reveal an observed smaller $O(d)$ scaling.}
    \label{fig:tv-dd}
\end{figure}

For $d$-dimensional Gaussian mixtures where $b(t,x)$ admits analytical solutions, we evaluate the empirical TV error without model training. Figure \ref{fig:ddvsh} demonstrates the error growth rate versus step size $h$, confirming our theoretical convergence rates. Figure \ref{fig:ddvsd1} and \ref{fig:ddvsd2} further examines the dimensional dependence at fixed $h$. While our theory establishes $O(hd^2)$ (Euler) and $O(h^2d^3)$ (Heun) error bounds, empirical observations suggest linear growth in both cases. This theory-experiment gap may indicate suboptimal bounds, warranting future investigation.

%% file: contents/conclusion.tex
\section{Conclusions}

This paper presented a finite-time analysis of discrete-time numerical implementations for ODEs derived within the stochastic interpolants framework. We established total variation distance error bounds for both the first-order forward Euler and second-order Heun's methods, quantifying the discrepancy between true and approximated target distributions. Furthermore, we analyzed the iteration complexity for both methods, elucidating their convergence rates. Numerical experiments corroborated our theoretical convergence findings. However, results from $d$-dimensional Gaussian mixtures suggest the potential for tighter error bounds, motivating future work to validate optimal bounds and develop analyses to achieve them.

%% file: contents/appendix-lemmas.tex
\section{Technical Lemmas}

In this section, we use the notation $x_t$ to denote the stochastic process 
$$x_t=I(t,x_0,x_1)+\gamma(t)z$$
defined in the stochastic interpolant, and use the notation $X_t$ to denote the process satisfying the equation $$\frac{\dif}{\dif t}X_t=b(t,X_t)$$
and $X_t\sim\rho(t)$.

Recall that in stochastic interpolants, we have defined $b(t,x)=\mathbb{E}[\partial_tI(t,x_0,x_1)+\dot{\gamma}z|x_t=x]$ and $s(t,x)=\nabla\log\rho(t,x)=\gamma^{-1}(t)\mathbb{E}[z|x_t=x]$. For convenience, we define $v(t,x)=\mathbb{E}[\partial_tI(t,x_0,x_1)|x_t=x]$.

\subsection{The Concrete Form of Derivatives of $s(t,x)$ and $v(t,x)$}

Before introducing the lemmas, we first define the notation $f_t=-\frac{\Vert x-I(t,x_0,x_1)\Vert^2}{2\gamma(t)^2}$, so that the conditional density $p(x_t|x_0,x_1)\propto\exp(f_t)$. The first lemma shows that how the jacobian matrices of $v$ and $s$ can be expressed.

\begin{lemma}
    $$\begin{aligned}
        \nabla_xv(t,x)&=\textnormal{Cov}(\partial_tI,\nabla_xf_t|x_t=x),\\
        \nabla_xs(t,x)&=\textnormal{Cov}(\nabla_xf_t,\nabla_xf_t|x_t=x)-\gamma^{-2}I_d.
    \end{aligned}$$
    \label{lem:vs-derivative}
\end{lemma}

\begin{proof}
    $$\nabla_xv(t,x)=\nabla_x\frac{\int\exp(f_t)\partial_tI\dif\nu}{\int\exp(f_t)\dif\nu}.$$
    The conditions for Lebesgue dominated convergence theorem can be easily checked, so the order of differential and integral can be alternated. Then
    $$\begin{aligned}
        \nabla_xv(t,x)&=\frac{\int\exp(f_t)(\nabla_x\partial_tI)\dif\nu}{\int\exp(f_t)\dif\nu}+\frac{\int\exp(f_t)\partial_tI\otimes\nabla_xf_t\dif\nu}{\int\exp(f_t)\dif\nu}\\
        &\quad-\frac{\int\exp(f_t)\partial_tI\dif\nu\otimes\int\exp(f_t)\nabla_xf_t\dif\nu}{\left[\int\exp(f_t)\dif\nu\right]^2}\\
        &=\textnormal{Cov}(\partial_tI,\nabla_xf_t|x_t=x).
    \end{aligned}$$
    Here the notation $\otimes$ represents the tensor product. Similarly,
    $$\begin{aligned}
        \nabla_xs(t,x)&=\frac{\int\exp(f_t)(\nabla_x^2f_t)\dif\nu}{\int\exp(f_t)\dif\nu}+\frac{\int\exp(f_t)\nabla_xf_t\otimes\nabla_xf_t\dif\nu}{\int\exp(f_t)\dif\nu}\\
        &\quad-\frac{\int\exp(f_t)\nabla_xf_t\dif\nu\otimes\int\exp(f_t)\nabla_xf_t\dif\nu}{\left[\int\exp(f_t)\dif\nu\right]^2}\\
        &=\textnormal{Cov}(\nabla_xf_t,\nabla_xf_t|x_t=x)+\nabla^2_x\left(-\frac{\Vert x-I\Vert^2}{2\gamma^2}\right),\\
        &=\textnormal{Cov}(\nabla_xf_t,\nabla_xf_t|x_t=x)-\gamma^{-2}I_d\textnormal{.}
    \end{aligned}$$
\end{proof}

In the above calculations, the key is to change the order of taking differential and integral, and rearrange the final formula into the form of conditional expectations. The higher order derivatives can be obtained similarly, but since the calculation is too long, we will give the results without showing the detailed proof. Lemmas \ref{lem:vs-derivative2} and \ref{lem:vs-derivative3} are similar as Lemma \ref{lem:vs-derivative}, while we consider the second-order and third-order derivatives.

\begin{lemma}
    $$\begin{aligned}
        \nabla_x^2 v(t,x)&=\mathbb{E}[(\partial_tI-v(t,x))\otimes(\nabla_xf_t-s(t,x))\otimes(\nabla_xf_t-s(t,x))|x_t=x],\\
        \nabla_x^2 s(t,x)&=\mathbb{E}[(\nabla_xf_t-s(t,x))\otimes(\nabla_xf_t-s(t,x))\otimes(\nabla_xf_t-s(t,x))|x_t=x],
    \end{aligned}$$
    Note that $v(t,x)$ and $s(t,x)$ are the conditional expectations of $\partial_tI$ and $\nabla_xf_t$, respectively.
    \label{lem:vs-derivative2}
\end{lemma}

\begin{lemma}
    $$\begin{aligned}
        \nabla_x^3v(t,x)&=\mathbb{E}[(\partial_tI-v(t,x))\otimes(\nabla_xf_t-s(t,x))^{\otimes3}|x_t=x]\\
        &\qquad-\mathcal{T}(\textnormal{Cov}(\partial_tI,\nabla_xf_t|x_t=x)\otimes\textnormal{Cov}(\nabla_xf_t,\nabla_xf_t|x_t=x)),\\
        \nabla_x^3 s(t,x)&=\mathbb{E}[(\nabla_xf_t-s(t,x))^{\otimes4}|x_t=x]\\
        &\qquad-\mathcal{T}(\textnormal{Cov}(\nabla_xf_t,\nabla_xf_t|x_t=x)^{\otimes2}).
    \end{aligned}$$
    Here, the notation $A^{\otimes n}$ represents the tensor product of $n$ tensors $A$. The operator $\mathcal{T}$ is defined as
    $$\mathcal{T}:\mathbb{R}^{d\times d\times d\times d}\to\mathbb{R}^{d\times d\times d\times d},\quad\mathcal{T}(X)=\mathcal{T}_{1234}(X)+\mathcal{T}_{1324}(X)+\mathcal{T}_{1423}(X),$$
    where
    $$\left[\mathcal{T}_{p_1p_2p_3p_4}(X)\right]_{i_1i_2i_3i_4}=X_{i_{p_1}i_{p_2}i_{p_3}i_{p_4}}.$$
    \label{lem:vs-derivative3}
\end{lemma}

Similarly to the derivatives with respect to $x$, the following lemmas (Lemmas \ref{lem:vs-time-derivative}, \ref{lem:vs-time-derivative2} and \ref{lem:vs-grad-time-derivative2}) calculates the derivatives of $v$, $s$ and their Jacobian matrices with respect to the time $t$.

\begin{lemma}
    $$\begin{aligned}
        \partial_t v(t,x)&=\textnormal{Cov}[\partial_tI,\partial_tf_t|x_t=x]+\mathbb{E}[\partial_t^2I|x_t=x],\\
        \partial_t s(t,x)&=\textnormal{Cov}[\nabla_xf_t,\partial_tf_t|x_t=x]+\mathbb{E}[\partial_t\nabla_xf_t|x_t=x],\\
        \partial_t\nabla_xv(t,x)&=\textnormal{Cov}(\partial_t^2I,\nabla_x f_t|x_t=x)+\textnormal{Cov}(\partial_tI,\partial_t\nabla_x f_t|x_t=x)\\
        &\qquad+\mathbb{E}[(\partial_tI-v(t,x))\otimes(\nabla_xf_t-s(t,x))\otimes(\partial_tf_t-m_{\partial_tf_t})|x_t=x],\\
        \partial_t\nabla_xs(t,x)&=\textnormal{Cov}(\partial_t\nabla_xf_t,\nabla_x f_t|x_t=x)+\textnormal{Cov}(\nabla_xf_t,\partial_t\nabla_x f_t|x_t=x)\\
        &\qquad+\mathbb{E}[(\nabla_xf_t-s(t,x))\otimes(\nabla_xf_t-s(t,x))\otimes(\partial_tf_t-m_{\partial_tf_t})|x_t=x].
    \end{aligned}$$
    In the above formula, we use the notation $m_{V}=\mathbb{E}[V|x_t=x]$ for simplicity.
    \label{lem:vs-time-derivative}
\end{lemma}

\begin{lemma}
    \label{lem:vs-time-derivative2}
    $$\begin{aligned}
        \partial_t^2 v(t,x)&=\mathbb{E}[\partial_t^3I|x_t=x]+2\textnormal{Cov}[\partial_t^2I,\partial_tf_t|x_t=x]+\textnormal{Cov}[\partial_tI,\partial_t^2f_t|x_t=x]\\
        &\quad+\mathbb{E}[(\partial_tI-v(t,x))\otimes(\partial_tf_t-m_{\partial_tf_t})\otimes(\partial_tf_t-m_{\partial_tf_t})|x_t=x],\\
        \partial_t^2 s(t,x)&=\mathbb{E}[\partial_t^2\nabla_xf_t|x_t=x]+2\textnormal{Cov}[\partial_t\nabla_xf_t,\partial_tf_t|x_t=x]+\textnormal{Cov}[\nabla_xf_t,\partial_t^2f_t|x_t=x]\\
        &\quad+\mathbb{E}[(\nabla_xf_t-s(t,x))\otimes(\partial_tf_t-m_{\partial_tf_t})\otimes(\partial_tf_t-m_{\partial_tf_t})|x_t=x],
    \end{aligned}$$
\end{lemma}

\begin{lemma}
    \label{lem:vs-grad-time-derivative2}
    $$\begin{aligned}
        \partial_t^2\nabla_xv(t,x)&=\textnormal{Cov}(\partial_t^3I,\nabla_x f_t|x_t=x)+2\textnormal{Cov}(\partial_t^2I,\partial_t\nabla_x f_t|x_t=x)\\
        &\qquad+\textnormal{Cov}(\partial_tI,\partial_t^2\nabla_x f_t|x_t=x)\\
        &\qquad+2\mathbb{E}[(\partial_t^2I-m_{\partial_t^2I})\otimes(\nabla_xf_t-s(t,x))\otimes(\partial_tf_t-m_{\partial_tf_t})|x_t=x]\\
        &\qquad+2\mathbb{E}[(\partial_tI-v(t,x))\otimes(\partial_t\nabla_xf_t-m_{\partial_t\nabla_xf_t})\otimes(\partial_tf_t-m_{\partial_tf_t})|x_t=x]\\
        &\qquad+\mathbb{E}[(\partial_tI-v(t,x))\otimes(\nabla_xf_t-m_{\nabla_xf_t})\otimes(\partial_t^2f_t-m_{\partial_t^2f_t})|x_t=x]\\
        &\qquad+\mathbb{E}[(\partial_tI-v(t,x))\otimes(\partial_tf_t-m_{\partial_tf_t})^{\otimes3}|x_t=x]\\
        &\qquad-\mathcal{T}(\textnormal{Cov}(\partial_tI,\partial_tf_t|x_t=x)\otimes\textnormal{Cov}(\partial_tf_t,\partial_tf_t|x_t=x))
    \end{aligned}$$
    $$\begin{aligned}
        \partial_t^2\nabla_xs(t,x)&=\textnormal{Cov}(\partial_t^2\nabla_xf_t,\nabla_x f_t|x_t=x)+2\textnormal{Cov}(\partial_t\nabla_xf_t,\partial_t\nabla_x f_t|x_t=x)\\
        &\qquad+\textnormal{Cov}(\nabla_xf_t,\partial_t^2\nabla_x f_t|x_t=x)\\
        &\qquad+2\mathbb{E}[(\partial_t\nabla_xf_t-m_{\partial_t\nabla_xf_t})\otimes(\nabla_xf_t-s(t,x))\otimes(\partial_tf_t-m_{\partial_tf_t})|x_t=x]\\
        &\qquad+2\mathbb{E}[(\nabla_xf_t-s(t,x))\otimes(\partial_t\nabla_xf_t-m_{\partial_t\nabla_xf_t})\otimes(\partial_tf_t-m_{\partial_tf_t})|x_t=x]\\
        &\qquad+\mathbb{E}[(\nabla_xf_t-s(t,x))\otimes(\nabla_xf_t-m_{\nabla_xf_t})\otimes(\partial_t^2f_t-m_{\partial_t^2f_t})|x_t=x]\\
        &\qquad+\mathbb{E}[(\nabla_xf_t-s(t,x))\otimes(\partial_tf_t-m_{\partial_tf_t})^{\otimes3}|x_t=x]\\
        &\qquad-\mathcal{T}(\textnormal{Cov}(\nabla_xf_t,\partial_tf_t|x_t=x)\otimes\textnormal{Cov}(\partial_tf_t,\partial_tf_t|x_t=x)).
    \end{aligned}$$
\end{lemma}

\subsection{Upper Bounds on the Derivatives}

We first provide a uniform upper bound to explain why Assumption \ref{assumption:lipschitz} is reasonable.

\begin{lemma}
    Suppose that for all $t\in[0,1]$, $P(\Vert I(t,x_0,x_1)\Vert\le R)=1$, i.e. the data is bounded. Then under Assumption 1, 
    $$\begin{aligned}
        \Vert\nabla b(t,x)\Vert_F&\lesssim\gamma(t)^{-4}R^2,\\
        \Vert\nabla^2b(t,x)\Vert_F&\lesssim\gamma(t)^{-6}R^3,\\
        \Vert\nabla(\nabla\cdot b(t,x))\Vert&\lesssim\gamma(t)^{-6}R^3.
    \end{aligned}$$
    \label{lem:bound-uniform}
\end{lemma}

\begin{proof}
    According to Lemma \ref{lem:vs-derivative}, since $\gamma\dot{\gamma}=O(1)$, for any $u\in\mathbb{R}^d$, by noticing that $\Vert x\otimes y\Vert_F=\Vert x\Vert\cdot\Vert y\Vert$,
    $$\begin{aligned}
        \Vert\nabla b(t,x)\Vert_F&=\Vert\nabla v(t,x)+\gamma(t)\dot{\gamma}(t)\nabla s(t,x)\Vert\\
        &\lesssim\mathbb{E}[(\Vert(\partial_tI-v(t,x))\Vert+\Vert\nabla_xf_t-s(t,x)\Vert)\cdot\Vert\nabla_xf_t-s(t,x)\Vert|x_t=x].
    \end{aligned}$$
    Since $\Vert\partial_tI\Vert\lesssim R$, $\nabla_xf_t-s(t,x)=-\frac{x-I}{\gamma(t)^2}-\mathbb{E}[-\frac{x-I}{\gamma(t)^2}|x_t=x]=\gamma(t)^{-2}(I-\mathbb{E}[I|x_t=x])\lesssim \gamma(t)^{-2}R$, we have
    $$\Vert\nabla b(t,x)\Vert\lesssim\gamma(t)^{-4}R^2.$$
    Similarly,
    $$\Vert\nabla^2b(t,x)\Vert_F\lesssim (\gamma(t)^{-2}R)^2(1+\gamma(t)^{-2})R\lesssim\gamma(t)^{-6}R^3.$$
    $$\Vert\nabla(\nabla\cdot b(t,x))\Vert\lesssim (\gamma(t)^{-2}R)^2(1+\gamma(t)^{-2})R\lesssim\gamma(t)^{-6}R^3.$$
\end{proof}

Lemma \ref{lem:bound-uniform} justifies the choice of Lipschitz constants in Assumption \ref{assumption:lipschitz}, where we set $L$ and $L^{3/2}$ as the Lipschitz constants for $\hat{b}(t,x)$ and its spatial derivative, respectively. When $R=O(\sqrt{d})$, as occurs when data are bounded in each dimension, it follows that $L=O(d)$.

The rest of this section analyzes expectation-based upper bounds, rather than uniform bounds with respect to $x$ and $t$.

\begin{lemma}
    For a Gaussian random variable $z\sim\mathcal{N}(0,I_d)$, for any constant $p\ge 2$, 
    $$\mathbb{E}\left[\Vert z\Vert^p\right]\le C(p)d^{p/2},$$
    where the constant $C(p)>0$ only depends on $p$.
    \label{lem:gaussian-moment}
\end{lemma}

\begin{proof}
    First, by Jensen's inequality,
    $$\Vert z\Vert^p=(\Vert z\Vert^2)^{p/2}=\left(\sum_{k=1}^d|z_i|^2\right)^{p/2}\le d^{p/2}\sum_{k=1}^d\frac{1}{p}|z_i|^p,$$
    so
    $$\mathbb{E}\left[\Vert z\Vert^p\right]\le d^{p/2}\mathbb{E}|z_1|^p.$$
    Here, $C(p)=\mathbb{E}|z_1|^p<\infty$ is a constant that only depends on $p$ since $z_1\sim N(0,1)$ is a standard 1-dimensional Gaussian variable.
\end{proof}

The following lemmas (Lemma \ref{lem:vs-moment} to Lemma \ref{lem:vs-grad-partialt-moment}) provide upper bounds on the moments of the time and spatial derivatives of $v$ and $s$. Notably, for these lemmas, the expectation is taken over $x_t\sim\rho(t)$, where recall that $x_t$ is the stochastic interpolant.

\begin{lemma}
    Under Assumption \ref{assumption:regularity}, for $p\ge 2$,
    $$\begin{aligned}
        \mathbb{E}\left[\Vert v(t,x_t)\Vert^p\right]&\lesssim\mathbb{E}\left[\Vert x_0-x_1\Vert^p\right],\\
        \mathbb{E}\left[\Vert s(t,x_t)\Vert^p\right]&\lesssim\gamma(t)^{-p}d^{p/2}.
    \end{aligned}$$
    \label{lem:vs-moment}
\end{lemma}

\begin{proof}
    By law of total probability and Jensen's inequality,
    $$\begin{aligned}
        \mathbb{E}\left[\Vert v(t,x_t)\Vert^p\right]&\le\mathbb{E}\left[\Vert\partial_tI\Vert^p\right]\lesssim\mathbb{E}\left[\Vert x_0-x_1\Vert^p\right],\\
        \mathbb{E}\left[\Vert s(t,x_t)\Vert^p\right]&\le\gamma(t)^{-p}\mathbb{E}\left[\Vert z\Vert^p\right]\lesssim\gamma(t)^{-p}d^{p/2}.
    \end{aligned}$$
\end{proof}

\begin{lemma}
    Under Assumption \ref{assumption:regularity}, for any $p\ge 1$, we have
    $$\begin{aligned}
        \mathbb{E}\left[\Vert\nabla_xv(t,x_t)\Vert_F^p\right]&\lesssim\gamma(t)^{-p}d^{p/2}\sqrt{\mathbb{E}\left[\Vert x_0-x_1\Vert^{2p}\right]},\\
        \mathbb{E}\left[\Vert \nabla_xs(t,x_t)\Vert_F^p\right]&\lesssim\gamma(t)^{-2p}d^{p}.
    \end{aligned}$$
    \label{lem:vs-grad-moment}
\end{lemma}

\begin{proof}
    $$\Vert\textnormal{Cov}(\partial_tI,\nabla_xf_t|x_t=x)\Vert_F\le\sqrt{\mathbb{E}\left[\Vert\partial_tI\Vert^2|x_t=x\right]\mathbb{E}\left[\Vert\nabla_xf_t\Vert^2|x_t=x\right]}.$$
    Then, by Jensen's inequalty and Cauchy-Schwarz inequality,
    $$\begin{aligned}
        \mathbb{E}\left[\Vert\nabla_xv(t,x_t)\Vert_F^p\right]
        &\lesssim\sqrt{\mathbb{E}\left[\Vert\partial_tI\Vert^{2p}\right]}\sqrt{\mathbb{E}\left[\Vert\nabla_xf_t\Vert^{2p}\right]}\\
        &\lesssim\gamma(t)^{-p}d^{p/2}\sqrt{\mathbb{E}\left[\Vert x_0-x_1\Vert^{2p}\right]}.
    \end{aligned}$$
    Similarly, 
    $$\begin{aligned}
        \mathbb{E}\left[\Vert\nabla_xs(t,x_t)\Vert_F^p\right]
        &\lesssim\gamma(t)^{-2p}d^{p/2}+\mathbb{E}\left[\Vert\nabla_xf_t\Vert^{2p}\right]\\
        &\lesssim\gamma(t)^{-2p}d^{p}.
    \end{aligned}$$
\end{proof}

\begin{lemma}
    Under Assumption \ref{assumption:regularity}, for any $p\ge 1$,
    $$\begin{aligned}
        \mathbb{E}\left[\Vert\partial_tv(t,x_t)\Vert^p\right]&\lesssim\mathbb{E}\left[\Vert x_0-x_1\Vert^{2p}\right]\gamma^{-p}d^{p/2}+\mathbb{E}\left[\Vert x_0-x_1\Vert^{2p}\right]^{1/2}\gamma^{-2p}d^p,\\
        \mathbb{E}\left[\Vert \partial_ts(t,x_t)\Vert^p\right]&\lesssim\gamma^{-3p}d^{3p/2}+\mathbb{E}\left[\Vert x_0-x_1\Vert^{2p}\right]^{1/2}\gamma^{-2p}d^p.
    \end{aligned}$$
    \label{lem:vs-partialt-moment}
\end{lemma}

\begin{lemma}
    Under Assumption \ref{assumption:regularity}, for any $p\ge 1$,
    $$\begin{aligned}
        \mathbb{E}\left[|\partial_t\nabla\cdot v(t,x_t)|^p\right]&\lesssim\gamma^{-3p}d^{3p/2}\mathbb{E}\left[\Vert x_0-x_1\Vert^{3p}\right]^{1/3}+\gamma^{-2p}d^{p}\mathbb{E}\left[\Vert x_0-x_1\Vert^{3p}\right]^{2/3},\\
        \mathbb{E}\left[|\partial_t\nabla\cdot s(t,x_t)|^p\right]&\lesssim\gamma^{-4p}d^{2p}+\gamma^{-3p}d^{3p/2}\mathbb{E}\left[\Vert x_0-x_1\Vert^{3p}\right]^{1/3}.
    \end{aligned}$$
    \label{lem:vs-grad-partialt-moment}
\end{lemma}

The proofs for the above two lemmas are omitted as they almost repeat the proof of Lemma \ref{lem:vs-grad-moment}.

Now, with the previous lemmas, we are ready to provide upper bounds on the time-derivatives of $b(t,X_t)$ (Lemmas \ref{lem:ode-b-derivative} and \ref{lem:ode-b-derivative2}) and its divergence $\nabla\cdot b(t,X_t)$ (Lemmas \ref{lem:ode-bdiv-derivative} and \ref{lem:ode-bdiv-derivative2}), where $X_t$ is the true solution of the ODE $\dif X_t=b(t,X_t)$ that satisfies $X_t\sim\rho(t)$. These results are later used to control the discretization error of numerical methods.

\begin{lemma}
    Under Assumption \ref{assumption:regularity}, for $p\ge 1$, 
    $$\mathbb{E}\left[\left\Vert\frac{\dif}{\dif t}b(t,X_t)\right\Vert^{p}\right]\lesssim\mathbb{E}\left[\Vert x_0-x_1\Vert^{3p}\right]^{2/3}\gamma(t)^{-p}d^{p/2}+\gamma(t)^{-3p}d^{3p/2}.$$
    \label{lem:ode-b-derivative}
\end{lemma}

\begin{proof}
    $$\begin{aligned}
        \frac{\dif}{\dif t}b(t,X_t)&=\partial_tb(t,X_t)+\nabla_xb(t,X_t)\cdot b(t,X_t)\\
        &=\partial_tv(t,X_t)+\partial_t\left[\dot{\gamma}\gamma s(t,X_t)\right]+\left[\nabla_xv(t,X_t)+\dot{\gamma}\gamma\nabla_x s(t,X_t)\right]\cdot b(t,X_t).
    \end{aligned}$$
    By Jensen's inequality, $\left(\sum_{k=1}^na_i\right)^p\le n^{p-1}\sum_{k=1}^n|a_i|^p$, so
    $$\begin{aligned}
        \mathbb{E}\left[\left\Vert\frac{\dif}{\dif t}b(t,X_t)\right\Vert^p\right]
        &\lesssim\mathbb{E}\left[\Vert\partial_tv(t,X_t)\Vert^p\right]\\
        &\quad+\mathbb{E}\left[\Vert\partial_ts(t,X_t)\Vert^p\right]+\mathbb{E}\left[\Vert s(t,X_t)\Vert^p\right]\\
        &\quad+\mathbb{E}\left[\left\Vert\nabla_xv(t,X_t)+\nabla_xs(t,X_t)\right\Vert^p\cdot\Vert b(t,X_t)\Vert^p\right]\\
        &\overset{\textnormal{(a)}}{\lesssim}\mathbb{E}\left[\Vert\partial_tv(t,X_t)\Vert^p\right]\\
        &\quad+\mathbb{E}\left[\Vert\partial_ts(t,X_t)\Vert^p\right]+\mathbb{E}\left[\Vert s(t,X_t)\Vert^p\right]\\
        &\quad+\mathbb{E}\left[\left\Vert\nabla_xv(t,X_t)+\nabla_xs(t,X_t)\right\Vert^{3p/2}\right]^{2/3}\cdot\mathbb{E}\left[\Vert b(t,X_t)\Vert^{3p}\right]^{1/3}\\
        &\overset{\textnormal{(b)}}{\lesssim}\mathbb{E}\left[\Vert x_0-x_1\Vert^{3p}\right]^{2/3}\gamma(t)^{-p}d^{p/2}+\gamma(t)^{-3p}d^{3p/2}.
    \end{aligned}$$
    The inquality (a) uses H\"older's inequality, while the inequality (b) uses the results of previous lemmas.
\end{proof}

\begin{lemma}
    Under Assumption \ref{assumption:regularity}, for $p\ge 1$, 
    $$\mathbb{E}\left[\left|\frac{\dif}{\dif t}\left(\nabla\cdot b(t,X_t)\right)\right|^{p}\right]\lesssim\mathbb{E}\left[\Vert x_0-x_1\Vert^{3p}\right]^{2/3}\gamma^{-2p}d^p+\gamma^{-4p}d^{2p},$$
    $$\mathbb{E}\left[\left\Vert\frac{\dif}{\dif t}\nabla b(t,X_t)\right\Vert^{p}\right]\lesssim\mathbb{E}\left[\Vert x_0-x_1\Vert^{3p}\right]^{2/3}\gamma^{-2p}d^p+\gamma^{-4p}d^{2p}.$$
    \label{lem:ode-bdiv-derivative}
\end{lemma}

\begin{proof}
    Similarly to the proof of Lemma \ref{lem:ode-b-derivative},
    $$\begin{aligned}
        \mathbb{E}\left[\left|\frac{\dif}{\dif t}\left(\nabla\cdot b(t,X_t)\right)\right|^{p}\right]
        &\lesssim\mathbb{E}\left[\left|\textnormal{tr}\left(\partial_t\nabla_x b(t,x)\right)\right|^p\right]\\
        &+\mathbb{E}\left[\left|\textnormal{tr}\left(\nabla^2b(t,X_t)[b(t,X_t)]\right)\right|^p\right].
    \end{aligned}$$
    For the first term on the right hand side, we can use Lemma \ref{lem:vs-grad-partialt-moment}. For the second term, note that for any four vectors $x,y,z,w$, we have 
    $$\left|\textnormal{tr}\left[(x\otimes y\otimes z)[w]\right]\right|=\left|\textnormal{tr}\left[(x\otimes y)\cdot(z\cdot w)\right]\right|\le\Vert x\Vert\cdot\Vert y\Vert\cdot\Vert z\Vert\cdot\Vert w\Vert.$$
    Hence, by the conditional expectation form in Lemma \ref{lem:vs-derivative2},
    $$\begin{aligned}
        \mathbb{E}\left[\left|\textnormal{tr}\left(\nabla^2b(t,X_t)[b(t,X_t)]\right)\right|^p\right]&\lesssim\mathbb{E}\left[\Vert x_0-x_1\Vert^{3p}\right]^{2/3}\gamma^{-2p}d^p+\gamma^{-4p}d^{2p}.
    \end{aligned}$$
    So,
    $$\mathbb{E}\left[\left|\frac{\dif}{\dif t}\left(\nabla\cdot b(t,X_t)\right)\right|^{p}\right]\lesssim\mathbb{E}\left[\Vert x_0-x_1\Vert^{3p}\right]^{2/3}\gamma^{-2p}d^p+\gamma^{-4p}d^{2p}.$$
    The case for the Jacobian matrix is the same except where we use $\Vert x\otimes y\Vert_F=\Vert x\Vert\cdot\Vert y\Vert$ instead of the inequality for the trace.
\end{proof}

\begin{lemma}
    Under Assumption \ref{assumption:regularity-2}, for $p\ge 1$,
    $$\mathbb{E}\left[\left\Vert\frac{\dif^2}{\dif t^2}b(t,X_t)\right\Vert^p\right]\lesssim\mathbb{E}\left[\Vert x_0-x_1\Vert^{5p}\right]^{3/5}d^{p}\gamma(t)^{-2p}+d^{5p/2}\gamma(t)^{-5p}.$$
    \label{lem:ode-b-derivative2}
\end{lemma}

\begin{proof}
    $$\begin{aligned}
        \frac{\dif^2}{\dif t^2}b(t,X_t)&=\partial_t^2b(t,X_t)+2\partial_t\nabla b(t,X_t)\cdot b(t,X_t)+\nabla b(t,X_t)\cdot\partial_tb(t,X_t)\\
        &\quad+\nabla^2b(t,X_t)\left[b(t,X_t)^{\otimes2}\right]+\left[\nabla b(t,X_t)\right]^2b(t,X_t)
    \end{aligned}$$
    Consider Lemma \ref{lem:vs-derivative} to \ref{lem:vs-time-derivative2} where the conditional expectation forms of the above terms are given, then use H\"older's inequality and Jensen's inequality to obtain the upper bound in the Lemma.
\end{proof}

\begin{lemma}
    Under Assumption \ref{assumption:regularity-2},
    $$\mathbb{E}\left[\left|\frac{\dif^2}{\dif t^2}\left(\nabla\cdot b(t,X_t)\right)\right|\right]\lesssim\mathbb{E}\left[\Vert x_0-x_1\Vert^{6}\right]^{1/2}d^{3/2}\gamma(t)^{-3}+d^3\gamma(t)^{-6},$$
    $$\mathbb{E}\left[\left\Vert\frac{\dif^2}{\dif t^2}\nabla b(t,X_t)\right\Vert_F\right]\lesssim\mathbb{E}\left[\Vert x_0-x_1\Vert^{6}\right]^{1/2}d^{3/2}\gamma(t)^{-3}+d^3\gamma(t)^{-6}.$$
    \label{lem:ode-bdiv-derivative2}
\end{lemma}

\begin{proof}
    $$\begin{aligned}
        \frac{\dif^2}{\dif t^2}\nabla b(t,X_t)&=\partial_t^2\nabla b(t,X_t)+2\partial_t\nabla^2b(t,X_t)[b(t,X_t)]+\nabla^2b(t,X_t)[\partial_tb(t,X_t)]\\
        &\quad+\nabla^3b(t,X_t)\left[b(t,X_t)^{\otimes2}\right]+\nabla^2b(t,X_t)\cdot\nabla b(t,X_t)\cdot b(t,X_t).
    \end{aligned}$$
    The rest is similar to the previous lemma.
\end{proof}

%% file: contents/appendix-euler.tex
\section{Proof of Theorem \ref{thm:euler}}

\subsection{Interpolation of the Discrete-time Process}

First, for the forward Euler solver, we apply the following interpolation to reformulate the process as a continuous-time process.
$$\hat{X}_t=F_{t_k\to t}(\hat{X}_{t_k}):=\hat{X}_{t_k}+(t-t_k)\hat{b}(t_k,\hat{X}_{t_k}),\quad\forall t\in[t_k,t_{k+1}),$$
that is,
$$\dif\hat{X}_t=\hat{b}(t_k,\hat{X}_{t_k})\dif t=\partial_tF_{t_k\to t}(\hat{X}_{t_k})\dif t.$$
To write $\hat{b}(t_k,\hat{X}_{t_k})$ as a function of $(t,\hat{X}_{t})$, we first need to show that $F_{t_k\to t}$ is a diffeomorphism from $\mathbb{R}^d$ to itself.

\begin{lemma}
    Under Assumption \ref{assumption:lipschitz}, suppose that the step size $h_k\le\frac{1}{2L}$, then $F_{t_k\to t}$ is a diffeomorphism, and
    $$\forall x\in\mathbb{R}^d,\quad\Vert\nabla F_{t_k\to t}(x)\Vert\le2,\quad\Vert\nabla F_{t_k\to t}^{-1}(x)\Vert\le2.$$
    \label{lem:diffeomorphism-euler}
\end{lemma}

\begin{proof}[Proof of Lemma \ref{lem:diffeomorphism-euler}]
    $$\nabla F_{t_k\to t}(x)=I_d+(t-t_k)\nabla\hat{b}(t_{k},x).$$
    So,
    $$\Vert\nabla F_{t_k\to t}(x)-I_d\Vert=(t-t_k)\Vert\nabla\hat{b}(t_{k},x)\Vert\le h_kL\le\frac{1}{2}<1.$$
    The above inequality shows that the Jacobi matrix $\nabla F_{t_k\to t}(x)$ is invertible, so $F_{t_k\to t}$ is a local diffeomprhism. At the same time, note that $\Vert F_{t_k\to t}(x)-F_{t_k\to t}(y)\Vert\ge\frac{1}{2}\Vert x-y\Vert$, so by Hadamard's global inverse function theorem, it is a global diffeomorphism on $\mathbb{R}^d$. Moreover, the matrix norm of the inverse of its Jacobi matrix satisfies $\Vert\nabla F_{t_k\to t}(x)^{-1}\Vert\le 2$.
\end{proof}

By Lemma \ref{lem:diffeomorphism-euler}, $\hat{X}_t$ satisfies the ODE
$$\dif\hat{X}_t=\tilde{b}(t,\hat{X}_t)\dif t=\hat{b}(t_k,F_{t_k\to t}^{-1}(\hat{X}_{t}))\dif t.$$
So according to Lemma \ref{lem:tv-dist},
$$\begin{aligned}
    \text{TV}(\rho(t_N),\hat{\rho}(t_N))&\le\text{TV}(\rho(t_0),\hat{\rho}(t_0))\\
    &\quad+\int_{t_0}^{t_N}\mathbb{E}_{X_t\sim\rho(t)}\left[|\nabla\cdot\tilde{b}(t,X_t)-\nabla\cdot b(t,X_t)|\right]\dif t\\
    &\quad+\int_{t_0}^{t_N}\mathbb{E}_{X_t\sim\rho(t)}\left[\Vert\tilde{b}(t,X_t)-b(t,X_t)\Vert\cdot\Vert\nabla\ln\rho(t,X_t)\Vert\right]\dif t.
\end{aligned}$$

In the following sections, we denote by $(X_t)_{t\in[t_0,t_N]}$ the true solution of the original ODE. When $t\in[t_k,t_{k+1})$ are given, we denote $z=F_{t_k\to t}^{-1}(X_t)$ for simplicity. The rest of the problem is to control the velocity error $\tilde{b}(t,X_t)-b(t,X_t)$ and the divergence error $\nabla\cdot\tilde{b}(t,X_t)-\nabla\cdot b(t,X_t)$.

\subsection{Controlling the Velocity Error}

By a simple triangle inequality,
$$\begin{aligned}
    \Vert\tilde{b}(t,X_t)-b(t,X_t)\Vert&\le\underbrace{\Vert\hat{b}(t_k,z)-\hat{b}(t_k,X_{t_k})\Vert}_{A}\\
    &+\underbrace{\Vert\hat{b}(t_k,X_{t_k})-b(t_k,X_{t_k})\Vert}_{B}\\
    &+\underbrace{\Vert b(t_k,X_{t_k})-b(t,X_{t_k})\Vert}_{C}.
\end{aligned}$$
Below we discuss the above terms respectively. For simplicity, we use the notation $\varepsilon_{1,k}(x)=\Vert\hat{b}(t_k,x)-b(t_k,x)\Vert$ to denote the error of $\hat{b}(t,x)$ at $(t_k,x)$.

For the term $A$, by Assumption \ref{assumption:lipschitz},
$$\begin{aligned}
    \Vert\hat{b}(t_k,z)-\hat{b}(t_k,X_{t_k})\Vert&\le L\Vert z-X_{t_k}\Vert\\
    &=L\Vert F_{t_k\to t}^{-1}(X_t)-F_{t_k\to t}^{-1}(F_{t_k\to t}(X_{t_k}))\Vert\\
    &\le 2L\Vert X_t-F_{t_k\to t}(X_{t_k})\Vert,
\end{aligned}$$
where the last inequality uses Lemma \ref{lem:diffeomorphism-euler} and the assumption that $h_k\le\frac{1}{2L}$. For the term $C$, we introduce the following lemma:

\begin{lemma}
    $$\Vert X_t-F_{t_k\to t}(X_{t_k})\Vert\le(t-t_k)\int_{t_k}^{t}\left\Vert\frac{\dif}{\dif s}\left(b(s,X_s)\right)\right\Vert\dif s+(t-t_k)\varepsilon_{1,k}(X_{t_k})$$
    \label{lem:euler-solution-error}
\end{lemma}

\begin{proof}[Proof of Lemma \ref{lem:euler-solution-error}]
    Consider taking derivatives and then integrate, we obtain that
    $$\begin{aligned}
        \Vert X_t-F_{t_k\to t}(X_{t_k})\Vert
        &=\left\Vert\int_{t_k}^{t}\left(\frac{\dif}{\dif s}X_s-\frac{\dif}{\dif s}F_{t_k\to s}(X_{t_k})\right)\dif s\right\Vert\\
        &=\left\Vert\int_{t_k}^{t}\left(b(s,X_s)-\hat{b}(t_k,X_{t_k})\right)\dif s\right\Vert\\
        &\le\int_{t_k}^{t}\left\Vert b(s,X_s)-b(t_k,X_{t_k})\right\Vert\dif s+(t-t_k)\varepsilon_{1,k}(X_{t_k})\\
        &\le\int_{t_k}^{t}\int_{t_k}^{s}\left\Vert\frac{\dif}{\dif u}\left(b(u,X_u)\right)\right\Vert\dif u+(t-t_k)\varepsilon_{1,k}(X_{t_k})\\
        &\le(t-t_k)\int_{t_k}^{t}\left\Vert\frac{\dif}{\dif s}\left(b(s,X_s)\right)\right\Vert\dif s+(t-t_k)\varepsilon_{1,k}(X_{t_k}).
    \end{aligned}$$
\end{proof}

Then, by Lemma \ref{lem:euler-solution-error}, term $A$ can be bounded below:
$$\begin{aligned}
    \Vert\hat{b}(t_k,z)-\hat{b}(t_k,X_{t_k})\Vert&\le2L(t-t_k)\int_{t_k}^{t}\left\Vert\frac{\dif}{\dif s}\left(b(s,X_s)\right)\right\Vert\dif s+\varepsilon_{1,k}(X_{t_k})\\
    &\le\int_{t_k}^{t}\left\Vert\frac{\dif}{\dif s}\left(b(s,X_s)\right)\right\Vert\dif s+\varepsilon_{1,k}(X_{t_k}),
\end{aligned}$$
where the last inequality uses $h_k\le\frac{1}{2L}$.

For the term $B$, it is just $\varepsilon_{1,k}(X_{t_k})$.

For the term $C$, apply the similar method as term $A$:
$$\Vert b(t_k,X_{t_k})-b(t,X_{t_k})\Vert\le\int_{t_k}^t\left\Vert\frac{\dif}{\dif s}b(s,X_s)\right\Vert\dif s.$$

To sum up, apply Lemma \ref{lem:ode-b-derivative},
$$\begin{aligned}
    &\qquad\mathbb{E}_{X_t\sim\rho(t)}\left[\Vert\tilde{b}(t,X_t)-b(t,X_t)\Vert\cdot\Vert\nabla\ln\rho(t,X_t)\Vert\right]\\
    &\le\mathbb{E}_{X_t\sim\rho(t)}\left[\Vert\tilde{b}(t,X_t)-b(t,X_t)\Vert^{4/3}\right]^{3/4}\cdot\mathbb{E}_{X_t\sim\rho(t)}\left[\Vert s(t,X_t)\Vert^4\right]^{1/4}\\
    &\lesssim\gamma(t)^{-1}d^{1/2}\mathbb{E}_{X_t\sim\rho(t)}\left[\left(\int_{t_k}^t\left\Vert\frac{\dif}{\dif s}b(s,X_s)\right\Vert\dif s+\varepsilon_{1,k}(X_{t_k})\right)^{4/3}\right]^{3/4}\\
    &\lesssim\gamma(t)^{-1}d^{1/2}\mathbb{E}_{X_t\sim\rho(t)}\left[(t-t_k)^{1/3}\int_{t_k}^t\left\Vert\frac{\dif}{\dif s}b(s,X_s)\right\Vert^{4/3}\dif s+\varepsilon_{1,k}(X_{t_k})^{4/3}\right]^{3/4}\\
    &\lesssim\gamma(t)^{-1}d^{1/2}\left[(t-t_k)^{1/3}\int_{t_k}^t\mathbb{E}\left[\left\Vert\frac{\dif}{\dif s}b(s,X_s)\right\Vert^{4/3}\right]+\mathbb{E}\left[\varepsilon_{1,k}(X_{t_k})^{4/3}\right]\right]^{3/4}\\
    &\lesssim\gamma(t)^{-1}d^{1/2}\mathbb{E}\left[\varepsilon_{1,k}(X_{t_k})^2\right]^{1/2}\\
    &\quad+\gamma(t)^{-1}d^{1/2}(t-t_k)\left[\mathbb{E}\left[\Vert x_0-x_1\Vert^4\right]^{1/2}\bar{\gamma}_k^{-1}d^{1/2}+\bar{\gamma}_k^{-3}d^{3/2}\right]\\
    &\lesssim\gamma(t)^{-1}d^{1/2}\mathbb{E}\left[\varepsilon_{1,k}(X_{t_k})^2\right]^{1/2}+h_k\bar{\gamma}_k^{-2}d\mathbb{E}\left[\Vert x_0-x_1\Vert^4\right]^{1/2}+h_k\bar{\gamma}_k^{-4}d^2.
\end{aligned}$$
Therefore,
$$\begin{aligned}
    &\qquad\int_{t_0}^{t_N}\mathbb{E}_{X_t\sim\rho(t)}\left[\Vert\tilde{b}(t,X_t)-b(t,X_t)\Vert\cdot\Vert\nabla\ln\rho(t,X_t)\Vert\right]\dif t\\
    &\lesssim d^{1/2}\int_{t_0}^{t_N}\gamma(t)^{-1}\mathbb{E}\left[\varepsilon_{1,k}(X_{t_k})^2\right]^{1/2}\dif t\\
    &\quad+\sum_{k=0}^{N-1}h_k^2\left[\bar{\gamma}_k^{-2}d\mathbb{E}\left[\Vert x_0-x_1\Vert^4\right]^{1/2}+\bar{\gamma}_k^{-4}d^2\right]\\
    &\overset{\text{(a)}}{\lesssim}d^{1/2}\sqrt{\int_{t_0}^{t_N}\gamma(t)^{-2}\dif t}\sqrt{\int_{t_0}^{t_N}\mathbb{E}\left[\varepsilon_{1,k}(X_{t_k})^2\right]^{1/2}\dif t}\\
    &\quad+\sum_{k=0}^{N-1}h_k^2\left[\bar{\gamma}_k^{-2}d\mathbb{E}\left[\Vert x_0-x_1\Vert^4\right]^{1/2}+\bar{\gamma}_k^{-4}d^2\right]\\
    &\overset{\text{(b)}}{\le}d^{1/2}\varepsilon_{\text{drift}}S(\gamma,t_0,t_N)^{1/2}+\sum_{k=0}^{N-1}h_k^2\left[\bar{\gamma}_k^{-2}d\mathbb{E}\left[\Vert x_0-x_1\Vert^4\right]^{1/2}+\bar{\gamma}_k^{-4}d^2\right],
\end{aligned}$$
where
\begin{equation}
    S(\gamma,t_0,t_N)=\int_{t_0}^{t_N}\gamma(t)^{-2}\dif t.
\end{equation}
For the above derivation, the inequality (a) applies Cauchy-Schwarz inequality, while the inequality (b) uses Assumption \ref{assumption:drift-error}.

\subsection{Controlling the Divergence Error}

First, according to the chain rule of derivatives,
$$\nabla\tilde{b}(t,X_t)=\nabla\hat{b}(t_k,z)\cdot\nabla F_{t_k\to t}(z)^{-1}.$$
Therefore, by applying triangle inequality,
$$\begin{aligned}
    \left|\nabla\cdot\tilde{b}(t,X_t)-\nabla\cdot b(t,X_t)\right|
    &\le\underbrace{\left|\text{tr}\left[\left[\nabla\hat{b}(t_k,z)-\nabla\hat{b}(t_k,X_{t_k})\right]\cdot\nabla F_{t_k\to t}(z)^{-1}\right]\right|}_{A}\\
    &+\underbrace{\left|\text{tr}\left[\left[\nabla\hat{b}(t_k,X_{t_k})-\nabla b(t_k,X_{t_k})\right]\cdot\nabla F_{t_k\to t}(z)^{-1}\right]\right|}_{B}\\
    &+\underbrace{\left|\text{tr}\left[\nabla b(t_k,X_{t_k})\cdot\left(\nabla F_{t_k\to t}(z)^{-1}-I_d\right)\right]\right|}_{C}\\
    &+\underbrace{\left|\nabla\cdot b(t_k,X_{t_k})-\nabla\cdot b(t,X_t)\right|}_{D}.
\end{aligned}$$
Now, we deal with the above terms respectively. Similarly to the velocity error, we use the notation $\varepsilon_{2,k}(x)=\Vert\nabla\hat{b}(t_k,x)-\nabla b(t_k,x)\Vert_F$.

First, since $\Vert\nabla F_{t_k\to t}(z)-I_d\Vert_F\le Lh_k\le\frac{1}{2}$,
$$\begin{aligned}
    \Vert\nabla F_{t_k\to t}(z)^{-1}-I_d\Vert_F&=\Vert(I_d+F_{t_k\to t}(z)-I_d)^{-1}-I_d\Vert_F\\
    &\le\Vert I_d+\sum_{i=1}^\infty(I_d-F_{t_k\to t}(z))^i-I_d\Vert_F\\
    &\le\sum_{i=1}^\infty(Lh_k)^i\le 2Lh_k.
\end{aligned}$$
So for the term $A$, 
$$\begin{aligned}
    &\quad\left|\text{tr}\left[\left[\nabla\hat{b}(t_k,z)-\nabla\hat{b}(t_k,X_{t_k})\right]\cdot\nabla F_{t_k\to t}(z)^{-1}\right]\right|\\
    &\le|\nabla\cdot\hat{b}(t_k,z)-\nabla\cdot\hat{b}(t_k,X_{t_k})|+2Lh_k\Vert\nabla\hat{b}(t_k,z)-\nabla\hat{b}(t_k,X_{t_k})\Vert_F\\
    &\le2L^{3/2}\Vert z-X_{t_k}\Vert\le4L^{3/2}\Vert X_t-F_{t_k\to t}(X_{t_k})\Vert\\
    &\le2L^{1/2}\int_{t_k}^t\left\Vert\frac{\dif}{\dif s}\left(b(s,X_s)\right)\right\Vert\dif s+2L^{1/2}\varepsilon_{1,k}(X_{t_k}), 
\end{aligned}$$
where the last inequality applies Lemma \ref{lem:euler-solution-error}.

For the term $B$,
$$\begin{aligned}
    &\quad\left|\text{tr}\left[\left[\nabla\hat{b}(t_k,X_{t_k})-\nabla b(t_k,X_{t_k})\right]\cdot\nabla F_{t_k\to t}(z)^{-1}\right]\right|\\
    &\le2\left\Vert\nabla\hat{b}(t_k,X_{t_k})-\nabla b(t_k,X_{t_k})\right\Vert_F\\
    &\le2\varepsilon_{2,k}(X_{t_k}).
\end{aligned}$$

For the term $C$,
$$\left|\text{tr}\left[\nabla b(t_k,X_{t_k})\cdot\left(\nabla F_{t_k\to t}(z)^{-1}-I_d\right)\right]\right|\le2Lh_k\Vert\nabla b(t_k,X_{t_k})\Vert_F.$$

For the term $D$, consider taking derivatives w.r.t. $t$:
$$\left|\nabla\cdot b(t_k,X_{t_k})-\nabla\cdot b(t,X_t)\right|\le\int_{t_k}^t\left|\frac{\dif}{\dif s}\left(\nabla\cdot b(s,X_s)\right)\right|\dif s.$$

Therefore, combining Lemmas \ref{lem:ode-bdiv-derivative}, \ref{lem:ode-b-derivative} and \ref{lem:vs-derivative}, we can get
$$\begin{aligned}.
    &\qquad\mathbb{E}\left[\left|\nabla\cdot\tilde{b}(t,X_t)-\nabla\cdot b(t,X_t)\right|\right]\\
    &\lesssim h_k\big[\mathbb{E}\left[\Vert x_0-x_1\Vert^{3}\right]^{2/3}(\bar{\gamma}_k^{-2}d+\bar{\gamma}_k^{-1}d^{1/2}L^{1/2})+\bar{\gamma}_k^{-4}d^{2}\\
    &\quad+L\bar{\gamma}_k^{-1}d^{1/2}\sqrt{\mathbb{E}\Vert x_0-x_1\Vert^2}+L\bar{\gamma}_k^{-2}d\big]\\
    &\quad+L^{1/2}\mathbb{E}\left[\varepsilon_{1,k}(X_{t_k})\right]+\mathbb{E}\left[\varepsilon_{2,k}(X_{t_k})\right].
\end{aligned}$$

Then, taking the integral and applying Assumption \ref{assumption:drift-error} and \ref{assumption:div-error}:
$$\begin{aligned}
    &\qquad\int_{t_0}^{t_N}\mathbb{E}_{X_t\sim\rho(t)}\left[|\nabla\cdot\tilde{b}(t,X_t)-\nabla\cdot b(t,X_t)|\right]\dif t\\
    &\lesssim\sum_{k=0}^{N-1}h_k^2\bigg[\mathbb{E}\left[\Vert x_0-x_1\Vert^{3}\right]^{2/3}\bar{\gamma}_k^{-2}d+\bar{\gamma}_k^{-4}d^{2}\\
    &\qquad+L\bar{\gamma}_k^{-1}d^{1/2}\sqrt{\mathbb{E}\Vert x_0-x_1\Vert^2}+L\bar{\gamma}_k^{-2}d\bigg]\\
    &\quad+\int_{t_0}^{t_N}\mathbb{E}\left[\varepsilon_{2,k}(X_{t_k})\right]\dif t+\int_{t_0}^{t_N}\mathbb{E}\left[\varepsilon_{1,k}(X_{t_k})\right]\dif t\\
    &\lesssim\varepsilon_{\text{div}}+L^{1/2}\varepsilon_{\text{drift}}+\sum_{k=0}^{N-1}h_k^2\bigg[\mathbb{E}\left[\Vert x_0-x_1\Vert^{3}\right]^{2/3}\bar{\gamma}_k^{-2}d+\bar{\gamma}_k^{-4}d^{2}\\
    &\qquad+L\bar{\gamma}_k^{-1}d^{1/2}\sqrt{\mathbb{E}\Vert x_0-x_1\Vert^2}+L\bar{\gamma}_k^{-2}d\bigg].
\end{aligned}$$

Combining Lemma \ref{lem:tv-dist}, the bound on velocity error and the bound on divergence error, we can obtain the error bound given in Theorem \ref{thm:euler}.

%% file: contents/appendix-heun.tex
\section{Proof of Theorem \ref{thm:heun}}

\subsection{Interpolation of the Discrete-time Process}

We first apply an interpolation on the Heun's method: $\hat{X}_t=G_{t_k\to t}(\hat{X_{t_k}})$, where
$$\begin{aligned}
    G_{t_k\to t}(x)&:=x+\int_{t_k}^t\left[\hat{b}(t_k,x)+\frac{s-t_k}{t_{k+1}-t_k}\left(\hat{b}(t_{k+1},G_{t_k\to t_{k+1}}(x))-\hat{b}(t_k,x)\right)\right]\dif s\\
    &=x+\left[(t-t_k)-\frac{(t-t_k)^2}{2(t_{k+1}-t_k)}\right]\hat{b}(t_k,x)+\frac{(t-t_k)^2}{2(t_{k+1}-t_k)}\hat{b}(t_{k+1},F_{t_k\to t_{k+1}}(x)).
\end{aligned}$$

\begin{lemma}
    Suppose Assumption \ref{assumption:lipschitz} holds, then if $h_k\le\frac{1}{4L}$, $G_{t_k\to t}$ is a diffeomprhism, and
    $$x\in\mathbb{R}^d,\quad\Vert\nabla G_{t_k\to t}(x)\Vert\le 2,\quad\Vert\nabla G_{t_k\to t}(x)^{-1}\Vert\le 2.$$
    \label{lem:diffeomorphism-heun}
\end{lemma}

\begin{proof}
    Similarly to the proof of Lemma \ref{lem:diffeomorphism-euler}, we only need to show that $\Vert\nabla G_{t_k\to t}(x)-I_d\Vert\le\frac{1}{2}$.
    $$\begin{aligned}
        \nabla G_{t_k\to t}(x)&=I_d+\left[(t-t_k)-\frac{(t-t_k)^2}{2(t_{k+1}-t_k)}\right]\nabla\hat{b}(t_k,x)\\
        &\quad+\frac{(t-t_k)^2}{2(t_{k+1}-t_k)}\nabla\hat{b}(t_{k+1},F_{t_k\to t_{k+1}}(x))\nabla F_{t_k\to t_{k+1}}(x)\\
        &=I_d+\left[(t-t_k)-\frac{(t-t_k)^2}{2(t_{k+1}-t_k)}\right]\nabla\hat{b}(t_k,x)\\
        &\quad+\frac{(t-t_k)^2}{2(t_{k+1}-t_k)}\nabla\hat{b}(t_{k+1},F_{t_k\to t_{k+1}}(x))\left(I_d+h_k\nabla\hat{b}(t_k,x))\right).
    \end{aligned}$$
    So,
    $$\Vert\nabla G_{t_k\to t}(x)-I_d\Vert\le h_kL+\frac{1}{2}h_kL\left(1+h_kL\right)<\frac{1}{2},$$
    which completes the proof.
\end{proof}

By Lemma \ref{lem:diffeomorphism-heun}, if we write $z=G_{t_k\to t}^{-1}(X_t)$, then
$$\frac{\dif}{\dif t}\hat{X}_t=\tilde{b}(t,\hat{X}_t):=\frac{t_{k+1}-t}{t_{k+1}-t_k}\hat{b}(t_k,z)+\frac{t-t_k}{t_{k+1}-t_k}\hat{b}(t_{k+1},F_{t_k\to t_{k+1}}(z)).$$
Then by Lemma \ref{lem:tv-dist},
$$\begin{aligned}
    \text{TV}(\rho(t_N),\hat{\rho}(t_N))&\le\text{TV}(\rho(t_0),\hat{\rho}(t_0))\\
    &\quad+\int_{t_0}^{t_N}\mathbb{E}_{X_t\sim\rho(t)}\left[|\nabla\cdot\tilde{b}(t,X_t)-\nabla\cdot b(t,X_t)|\right]\dif t\\
    &\quad+\int_{t_0}^{t_N}\mathbb{E}_{X_t\sim\rho(t)}\left[\Vert\tilde{b}(t,X_t)-b(t,X_t)\Vert\cdot\Vert\nabla\ln\rho(t,X_t)\Vert\right]\dif t.
\end{aligned}$$

\subsection{Controlling the Velocity Error}

Still, use the notation $\varepsilon_{1,k}(x)=\Vert\hat{b}(t_k,x)-b(t_k,x)\Vert$, then we can control $\tilde{b}(t,x)-b(t,x)$ in the following way:
$$\begin{aligned}
    &\qquad\tilde{b}(t,X_t)-b(t,X_t)\\&=\hat{b}(t_k,z)+\frac{t-t_k}{t_{k+1}-t_k}\left(\hat{b}(t_{k+1},F_{t_k\to t_{k+1}}(z))-\hat{b}(t_k,z)\right)-b(t,x)\\
    &=\underbrace{\frac{t_{k+1}-t}{t_{k+1}-t_k}\left[\hat{b}(t_k,z)-\hat{b}(t_k,X_{t_k})\right]+\frac{t-t_k}{t_{k+1}-t_k}\left[\hat{b}(t_{k+1},F_{t_k\to t_{k+1}}(z))-\hat{b}(t_{k+1},X_{t_{k+1}})\right]}_{A:\text{bias error}}\\
    &\quad+\underbrace{\frac{t_{k+1}-t}{t_{k+1}-t_k}\left[\hat{b}(t_k,X_{t_k})-b(t_k,X_{t_k})\right]+\frac{s-t_k}{t_{k+1}-t_k}\left[\hat{b}(t_{k+1},X_{t_{k+1}})-b(t_{k+1},X_{t_{k+1}})\right]}_{B:\text{estimation error}}\\
    &\quad+\underbrace{b(t_k,X_{t_k})+\frac{t-t_k}{t_{k+1}-t_k}\left(b(t_{k+1},X_{t_{k+1}})-b(t_k,X_{t_k})\right)-b(t,X_t)}_{C:\text{discretization error}}.
\end{aligned}$$

Below we bound the terms respectively. First, for the bias error $A$, 
$$\begin{aligned}
    A&\le\frac{t_{k+1}-t}{t_{k+1}-t_k}L\Vert z-X_{t_k}\Vert+\frac{t-t_k}{t_{k+1}-t_k}L\Vert F_{t_k\to t_{k+1}}(z)-X_{t_{k+1}}\Vert\\
    &\le\frac{t_{k+1}-t}{t_{k+1}-t_k}L\Vert z-X_{t_k}\Vert\\
    &\quad+\frac{t-t_k}{t_{k+1}-t_k}L\left(\Vert F_{t_k\to t_{k+1}}(z)-F_{t_k\to t_{k+1}}(X_{t_k})\Vert+\Vert F_{t_k\to t_{k+1}}(X_{t_k})-X_{t_{k+1}}\Vert\right)\\
    &\le 2L\Vert z-X_{t_k}\Vert+L(t-t_k)\int_{t_k}^{t_{k+1}}\left\Vert\frac{\dif}{\dif s}b(s,X_s)\right\Vert\dif s+L(t-t_k)\varepsilon_{1,k}(X_{t_k}).
\end{aligned}$$
The last step uses Lemma \ref{lem:euler-solution-error}.

For $\Vert z-X_{t_k}\Vert$, we can observe that
$$\Vert z-X_{t_k}\Vert\le\Vert G_{t_k\to t}^{-1}(X_t)-G_{t_k\to t}^{-1}(G_{t_k\to t}(X_{t_k}))\Vert\le 2\Vert X_t-G_{t_k\to t}(X_{t_k})\Vert.$$
Now, we introduce the following lemma:

\begin{lemma}
    When $h_k\le\frac{1}{L}$, suppose Assumption \ref{assumption:lipschitz} holds,
    $$\begin{aligned}
        \Vert X_t-G_{t_k\to t}(X_{t_k})\Vert&\lesssim(t-t_k)\varepsilon_{1,k}(X_{t_k})+(t-t_k)\varepsilon_{1,k+1}(X_{t_{k+1}})\\
        &+(t-t_k)^2L\int_{t_k}^{t_{k+1}}\left\Vert\frac{\dif}{\dif s}b(s,X_s)\right\Vert\dif s\\
        &+(t-t_k)(t_{k+1}-t_k)\int_{t_k}^{t_{k+1}}\left\Vert\frac{\dif^2}{\dif s^2}b(s,X_s)\right\Vert\dif s.
    \end{aligned}$$
    \label{lem:heun-solution-error}
\end{lemma}

\begin{proof}[Proof of Lemma \ref{lem:heun-solution-error}]
    First partition the error into several parts.
    $$\begin{aligned}
        &\qquad\Vert F_t(X_{t_k})-X_t\Vert\\
        &\le\left\Vert F_t(X_{t_k})-X_{t_k}-(X_t-X_{t_k})\right\Vert\\
        &\le\underbrace{\left[(t-t_k)-\frac{(t-t_k)^2}{2(t_{k+1}-t_k)}\right]\cdot\left\Vert\hat{b}(t_k,X_{t_k})-b(t_k,X_{t_k})\right\Vert}_{(\text{i})}\\
        &\quad+\underbrace{\frac{(t-t_k)^2}{2(t_{k+1}-t_k)}\cdot\left\Vert\hat{b}(t_{k+1},X_{t_{k+1}})-b(t_{k+1},X_{t_{k+1}})\right\Vert}_{\text{(ii)}}\\
        &\quad+\underbrace{\frac{(t-t_k)^2}{2(t_{k+1}-t_k)}\cdot\left\Vert\hat{b}(t_{k+1},F_{t_k\to t_{k+1}}(X_{t_k}))-\hat{b}(t_{k+1},X_{t_{k+1}})\right\Vert}_{\text{(iii)}}\\
        &\quad+\underbrace{\left\Vert\int_{t_k}^{t}\left[b(t_k,X_{t_k})+\frac{s-t_k}{t_{k+1}-t_k}(b(t_{k+1},X_{t_{k+1}})-b(t_k,X_{t_k}))-b(s,X_s)\right]\dif s\right\Vert}_{\text{(iv)}}.
    \end{aligned}$$
    We then need to provide upper bounds for the parts. First,
    $$\text{(i)}+\text{(ii)}\le2(t-t_k)\varepsilon_{1,k}(X_{t_k})+(t-t_k)\varepsilon_{1,k+1}(X_{t_{k+1}}).$$
    For the third term, apply Lemma \ref{lem:euler-solution-error},
    $$\begin{aligned}
        \text{(iii)}&\le\frac{(t-t_k)^2}{(t_{k+1}-t_k)}L\Vert F_{t_k\to t_{k+1}}(X_{t_k})-X_{t_{k+1}}\Vert\\
        &\le(t-t_k)^2L\int_{t_k}^{t_{k+1}}\left\Vert\frac{\dif}{\dif s}b(s,X_s)\right\Vert\dif s+(t-t_k)^2L\varepsilon_{1,k}(X_{t_k}).
    \end{aligned}$$
    For the last term, note that
    $$\begin{aligned}
        &\qquad \frac{t-t_k}{t_{k+1}-t_k}\left[b(t_{k+1},X_{t_{k+1}})-b(t_k,X_{t_k})\right]-\left[b(t,X_t)-b(t_k,X_{t_k})\right]\\
        &=\frac{t-t_k}{t_{k+1}-t_k}\int_{t_k}^{t_{k+1}}\frac{\dif}{\dif s}b(s,X_s)\dif s-\int_{t_k}^t\frac{\dif}{\dif s}b(s,X_s)\dif s\\
        &=\frac{t-t_k}{t_{k+1}-t_k}\int_{t_k}^{t_{k+1}}\left[\frac{\dif}{\dif s}b(s,X_s)-\frac{\dif}{\dif t}b(t,X_t)\right]\dif s-\int_{t_k}^t\left[\frac{\dif}{\dif s}b(s,X_s)-\frac{\dif}{\dif t}b(t,X_t)\right]\dif s\\
        &=\frac{t-t_k}{t_{k+1}-t_k}\int_{t}^{t_{k+1}}\dif s\int_t^s\frac{\dif^2}{\dif u^2}b(u,X_u)\dif u+\frac{t_{k+1}-t}{t_{k+1}-t_k}\int_{t_k}^{t}\dif s\int_s^t\frac{\dif^2}{\dif u^2}b(u,X_u)\dif u.
    \end{aligned}$$
    So,
    $$\begin{aligned}
        \text{(iv)}\le(t-t_k)(t_{k+1}-t_k)\int_{t_k}^{t_{k+1}}\left\Vert\frac{\dif^2}{\dif s^2}b(s,X_s)\right\Vert\dif s.
    \end{aligned}$$
    The proof is finished by directly combining the four parts.
\end{proof}

By Lemma \ref{lem:heun-solution-error}, the bias error
$$\begin{aligned}
    A&\lesssim L(t-t_k)\varepsilon_{1,k}(X_{t_k})+L(t-t_k)\varepsilon_{1,k+1}(X_{t_{k+1}})\\
    &\quad+L(t-t_k)\int_{t_k}^{t_{k+1}}\left\Vert\frac{\dif}{\dif s}b(s,X_s)\right\Vert\dif s\\
    &\quad+(t-t_k)Lh_k\int_{t_k}^{t_{k+1}}\left\Vert\frac{\dif^2}{\dif s^2}b(s,X_s)\right\Vert\dif s.
\end{aligned}$$

For the estimation error $B$, clearly,
$$\Vert B\Vert\le\frac{t_{k+1}-t}{t_{k+1}-t_k}\varepsilon_{1,k}(X_{t_k})+\frac{t-t_{k}}{t_{k+1}-t_k}\varepsilon_{1,k+1}(X_{t_{k+1}}).$$

For the discretization error $C$, according to the proof of Lemma \ref{lem:heun-solution-error}, we obtain that
$$\Vert C\Vert\lesssim (t_{k+1}-t_k)\int_{t_k}^{t_{k+1}}\left\Vert\frac{\dif^2}{\dif s^2}b(s,X_s)\right\Vert\dif s.$$

Combining the above three terms, and apply Lemmas \ref{lem:ode-b-derivative} and \ref{lem:ode-b-derivative2}, then we get
$$\begin{aligned}
    &\qquad\mathbb{E}_{X_t\sim\rho(t)}\left[\Vert\tilde{b}(t,X_t)-b(t,X_t)\Vert\cdot\Vert\nabla\ln\rho(t,X_t)\Vert\right]\\
    &\lesssim\mathbb{E}_{X_t\sim\rho(t)}\left[\Vert\tilde{b}(t,X_t)\Vert^{6/5}\right]^{5/6}\cdot\mathbb{E}_{X_t\sim\rho(t)}\left[\Vert s(t,X_t)\Vert^6\right]^{1/6}\\
    &\lesssim\gamma(t)^{-1}d^{1/2}\left[\mathbb{E}\left[\Vert\varepsilon_{1,k}(X_{t_k})\Vert^2\right]^{1/2}+\mathbb{E}\left[\Vert\varepsilon_{1,k}(X_{t_{k+1}})\Vert^2\right]^{1/2}\right]\\
    &\quad+\gamma(t)^{-1}d^{1/2}Lh_k\left[(t-t_k)^{1/5}\int_{t_k}^t\left\Vert\frac{\dif}{\dif s}b(s,X_s)\right\Vert^{6/5}\dif s\right]^{5/6}\\
    &\quad+\gamma(t)^{-1}d^{1/2}h_k\left[(t-t_k)^{1/5}\int_{t_k}^t\left\Vert\frac{\dif^2}{\dif s^2}b(s,X_s)\right\Vert^{6/5}\dif s\right]^{5/6}\\
    &\lesssim\gamma(t)^{-1}d^{1/2}\left[\mathbb{E}\left[\varepsilon_{1,k}(X_{t_k})^2\right]^{1/2}+\mathbb{E}\left[\varepsilon_{1,k}(X_{t_{k+1}})^2\right]^{1/2}\right]\\
    &\quad+\bar{\gamma}_k^{-3}d^{3/2}h_k^2\mathbb{E}\left[\Vert x_0-x_1\Vert^6\right]^{1/2}+d^3\bar{\gamma}_k^{-6}\\
    &\quad+L\bar{\gamma}_k^{-2}d\mathbb{E}\left[\Vert x_0-x_1\Vert^{6}\right]^{1/3}+Ld^2\bar{\gamma}_k^{-4}.
\end{aligned}$$
Integrate and write $M=\max\{\mathbb{E}\left[\Vert x_0-x_1\Vert^6\right]^{1/3},L,d\}$, then
$$\begin{aligned}
    &\int_{t_0}^{t_N}\mathbb{E}_{X_t\sim\rho(t)}\left[\Vert\tilde{b}(t,X_t)-b(t,X_t)\Vert\cdot\Vert\nabla\ln\rho(t,X_t)\Vert\right]\dif t\\
    &\lesssim d^{1/2}\varepsilon_{\text{drift}}S(\gamma,t_0,t_N)^{1/2}+\sum_{k=0}^{N-1}h_k^3\bigg[\bar{\gamma}_k^{-6}d^3+M^3\bar{\gamma}_k^{-4}\bigg].
\end{aligned}$$

\subsection{Controlling the Divergence Error}

Use the notation $\varepsilon_{2,k}(x)=\Vert\nabla\hat{b}(t_k,x)-\nabla b(t_k,x)\Vert_F$. First,
$$\begin{aligned}
    \nabla\cdot\tilde{b}(t,X_t)-\nabla\cdot b(t,X_t)&=\text{tr}\bigg[\frac{t_{k+1}-t}{t_{k+1}-t_k}\nabla\hat{b}(t_k,z)\nabla G_{t_k\to t}(z)^{-1}\\
    &\quad+\frac{t-t_k}{t_{k+1}-t_k}\nabla\hat{b}(t_{k+1},F_{t_k\to t_{k+1}}(z))\nabla F_{t_k\to t_{k+1}}(z)\nabla G_{t_k\to t}(z)^{-1}\\
    &\quad-\nabla b(t,X_t)\nabla G_{t_k\to t}(z)\nabla G_{t_k\to t}(z)^{-1}\bigg].
\end{aligned}$$
Then, by Lemma \ref{lem:diffeomorphism-heun}, 
$$\begin{aligned}
    \left|\nabla\cdot\tilde{b}(t,X_t)-\nabla\cdot b(t,X_t)\right|&\lesssim\bigg|\text{tr}\bigg[\frac{t_{k+1}-t}{t_{k+1}-t_k}\nabla\hat{b}(t_k,z)\\
    &\quad+\frac{t-t_k}{t_{k+1}-t_k}\nabla\hat{b}(t_{k+1},F_{t_k\to t_{k+1}}(z))\nabla F_{t_k\to t_{k+1}}(z)\\
    &\quad-\nabla b(t,X_t)\nabla G_{t_k\to t}(z)\bigg]\bigg|\\
    &=\left|\text{tr}\left[D_{\text{bias}}+D_{\text{est}}+D_{\text{dis}}\right]\right|,
\end{aligned}$$
where
$$\begin{aligned}
    D_{\text{bias}}&=\frac{t_{k+1}-t}{t_{k+1}-t_k}\left[\nabla\hat{b}(t_k,z)-\nabla\hat{b}(t_k,X_{t_k})\right]\\
    &\quad+\frac{t-t_k}{t_{k+1}-t_k}\bigg[\nabla\hat{b}(t_{k+1},F_{t_k\to t_{k+1}}(z))\nabla F_{t_k\to t_{k+1}}(z)-\nabla\hat{b}(t_{k+1},X_{t_{k+1}})\nabla G_{t_k\to t_{k+1}}(X_{t_k})\bigg]\\
    &\quad-\left[\nabla b(t,X_t)\nabla G_{t_k\to t}(z)-\nabla b(t,X_t)\nabla G_{t_k\to t}(X_{t_k})\right],\\
    D_{\text{est}}&=\frac{t_{k+1}-t}{t_{k+1}-t_k}\left[\nabla\hat{b}(t_k,X_{t_k})-\nabla b(t_k,X_{t_k})\right]\\
    &\quad+\frac{t-t_k}{t_{k+1}-t_k}\bigg[\nabla\hat{b}(t_{k+1},X_{t_{k+1}})\nabla G_{t_k\to t_{k+1}}(X_{t_k})-\nabla b(t_{k+1},X_{t_{k+1}})\nabla G_{t_k\to t_{k+1}}(X_{t_k})\bigg],\\
    D_{\text{dis}}&=\frac{t_{k+1}-t}{t_{k+1}-t_k}\nabla b(t_k,X_{t_k})\nabla F_{t_k\to t_{k}}(X_{t_k})+\frac{t-t_k}{t_{k+1}-t_k}\nabla b(t_{k+1},X_{t_{k+1}})\nabla G_{t_k\to t_{k+1}}(X_{t_k})\\
    &\quad-\nabla b(t,X_t)\nabla F_{t_k\to t}(X_{t_k}).
\end{aligned}$$

\paragraph{Bias Error}

For $D_\text{bias}$, with Assumption \ref{assumption:lipschitz},
$$\begin{aligned}
    \left|\text{tr}[D_{\text{bias}}]\right|&\le\frac{t_{k+1}-t}{t_{k+1}-t_k}L^{3/2}\Vert z-X_{t_k}\Vert\\
    &\quad+\frac{t-t_k}{t_{k+1}-t_k}\left\Vert\nabla\hat{b}(t_{k+1},F_{t_k\to t_{k+1}}(z))-\nabla\hat{b}(t_{k+1},X_{t_{k+1}})\right\Vert_F\cdot\Vert\nabla F_{t_k\to t_{k+1}}(z)\Vert_F\\
    &\quad+\frac{t-t_k}{t_{k+1}-t_k}\Vert\nabla\hat{b}(t_{k+1},X_{t_{k+1}})\Vert_F\cdot\left\Vert\nabla F_{t_k\to t_{k+1}}(z)-\nabla G_{t_k\to t_{k+1}}(X_{t_k})\right\Vert_F\\
    &\quad+\Vert\nabla b(t,X_{t})\Vert_F\cdot\left\Vert\nabla G_{t_k\to t}(X_{t_k})-\nabla G_{t_k\to t}(z)\right\Vert_F\\
    &\overset{\text{(a)}}{\lesssim}\frac{t_{k+1}-t}{t_{k+1}-t_k}L^{3/2}\Vert z-X_{t_k}\Vert+\frac{t-t_k}{t_{k+1}-t_k}L^{3/2}\left\Vert F_{t_k\to t_{k+1}}(z)-X_{t_{k+1}}\right\Vert\\
    &\quad+\frac{t-t_k}{t_{k+1}-t_k}L\bigg\Vert h_k\nabla\hat{b}(t_k,z)-\frac{1}{2}h_k\nabla\hat{b}(t_k,X_{t_k})\\
    &\qquad-\frac{1}{2}h_k\nabla\hat{b}(t_{k+1},F_{t_k\to t_{k+1}}(X_{t_k}))\nabla F_{t_k\to t_{k+1}}(X_{t_k})\bigg\Vert_F\\
    &\quad+\Vert\nabla b(t,X_t)\Vert_F\cdot L^{3/2}h_k\Vert X_{t_k}-z\Vert\\
    &\overset{\text{(b)}}{\lesssim}\Vert z-X_{t_k}\Vert L^{3/2}\left(1+h_k\Vert\nabla b(t,X_t)\Vert_F\right)+L^{3/2}\Vert F_{t_k\to t_{k+1}}(X_{t_k})-X_{t_{k+1}}\Vert\\
    &\quad+Lh_k\Vert\nabla F_{t_k\to t_{k+1}}(X_{t_k})-I_d\Vert\cdot\Vert\nabla\hat{b}(t_{k+1},X_{t_{k+1}})\Vert_F\\
    &\quad+Lh_k\left\Vert\nabla\hat{b}(t_k,X_{t_k})-\nabla\hat{b}(t_{k+1},X_{t_{k+1}})\right\Vert_F\\
    &\overset{\text{(c)}}{\lesssim} L^{3/2}(1+h_k\Vert\nabla b(t,X_t)\Vert_F)\bigg[h_k\varepsilon_{1,k}(X_{t_k})+h_k\varepsilon_{1,k+1}(X_{t_{k+1}})\\
    &\quad+h_k^2L\int_{t_k}^{t_{k+1}}\left\Vert\frac{\dif}{\dif s}b(s,X_s)\right\Vert\dif s+h_k^2\int_{t_k}^{t_{k+1}}\left\Vert\frac{\dif^2}{\dif s^2}b(s,X_s)\right\Vert\dif s\bigg]\\
    &\quad+L^{3/2}h_k\int_{t_k}^{t_{k+1}}\left\Vert\frac{\dif}{\dif s}b(s,X_s)\right\Vert\dif s\\
    &\quad+Lh_k\left(\varepsilon_{2,k}(X_{t_{k}})+\varepsilon_{2,k+1}(X_{t_{k+1}})\right)\\
    &\quad+Lh_k\int_{t_k}^{t_{k+1}}\left\Vert\frac{\dif}{\dif s}\left(\nabla b(s,X_s)\right)\right\Vert_F\dif s.
\end{aligned}$$
The inequality (a) applies Assumption \ref{assumption:lipschitz}; the inequality (b) rearranges the terms and applies triangle inequalities; the inequality (c) expands the differences into the form of integrals by Lemma \ref{lem:euler-solution-error}, \ref{lem:heun-solution-error}, and we control the difference $\left\Vert\nabla\hat{b}(t_k,X_{t_k})-\nabla\hat{b}(t_{k+1},X_{t_{k+1}})\right\Vert_F$ by controlling $\left\Vert\nabla b(t_k,X_{t_k})-\nabla b(t_{k+1},X_{t_{k+1}})\right\Vert_F$.

Notably, in the first term, when $dh_k\bar{\gamma}_k^{-2}\lesssim 1$ and $h_kd^{1/2}\bar{\gamma}_k^{-1}\mathbb{E}\left[\Vert x_0-x_1\Vert^p\right]^{1/p}\lesssim 1$ ($p>1$ is a small constant), we have $h_k\mathbb{E}\left\Vert\nabla b(t,X_t)\Vert^p\right]^{1/p}\lesssim 1$,
so based on the assumptions on $h_k$, Lemmas \ref{lem:ode-b-derivative}, \ref{lem:ode-b-derivative2}, \ref{lem:ode-bdiv-derivative} and H\"older's inequality,
$$\begin{aligned}
    \mathbb{E}\left[\left|\text{tr}[D_{\text{bias}}]\right|\right]
    &\lesssim L^{3/2}h_k^2\left(\mathbb{E}\left[\Vert x_0-x_1\Vert^6\right]^{1/3}\bar{\gamma}_k^{-1}d^{1/2}+\bar{\gamma}^{-3}d^{2}\right)\\
    &\quad+L^{3/2}h_k^3\left[\mathbb{E}\left[\Vert x_0-x_1\Vert^6\right]^{1/2}d\bar{\gamma}_k^{-2}+d^{5/2}\bar{\gamma}_k^{-5}\right]\\
    &\quad+Lh_k^2\left[\mathbb{E}\left[\Vert x_0-x_1\Vert^6\right]^{1/3}d\bar{\gamma}_k^{-2}+d^2\bar{\gamma}_k^{-4}\right]\\
    &\quad+L^{3/2}h_k\left(\mathbb{E}[\varepsilon_{1,k}(X_{t_k})]+\mathbb{E}[\varepsilon_{1,k+1}(X_{t_{k+1}})]\right)\\
    &\quad+Lh_k\left(\mathbb{E}[\varepsilon_{2,k}(X_{t_{k}})]+\mathbb{E}[\varepsilon_{2,k+1}(X_{t_{k+1}})]\right).
\end{aligned}$$

\paragraph{Estimation Error}

For $D_{\text{est}}$,
$$\begin{aligned}
    |\text{tr}[D_{\text{estimate}}]|&\le\frac{t_{k+1}-t}{t_{k+1}-t_k}\varepsilon_{2,k}(X_{t_k})+\frac{t-t_k}{t_{k+1}-t_k}\varepsilon_{2,k+1}(X_{t_{k+1}})\cdot\Vert\nabla G_{t_k\to t_{k+1}}(X_{t_k})\Vert\\
    &\lesssim\varepsilon_{2,k}(X_{t_k})+\varepsilon_{2,k+1}(X_{t_{k+1}}).
\end{aligned}$$

\paragraph{Discretization Error}

For $D_{\text{dis}}$, similarly to the analysis on the discretization error of $b(t,X_t)$,
$$|\text{tr}[D_{\text{dis}}]|\le(t_{k+1}-t_k)\int_{t_k}^{t_{k+1}}\left|\frac{\dif^2}{\dif s^2}\text{tr}\left[\nabla b(s,X_s)\nabla G_{t_k\to s}(X_{t_k})\right]\right|\dif s.$$
We now explicitly write the derivatives of $\nabla G_{t_k\to s}$ below for further discussion:
$$\begin{aligned}
    \nabla G_{t_k\to t}(x)&=I_d+\left[(t-t_k)-\frac{(t-t_k)^2}{2h_k}\right]\nabla\hat{b}(t_k,x)\\
    &\quad+\frac{(t-t_k)^2}{2h_k}\nabla\hat{b}(t_{k+1},F_{t_k\to t_{k+1}}(x))\cdot(I_d+h_k\nabla\hat{b}(t_k,x)),\\
    \frac{\dif}{\dif t}\nabla G_{t_k\to t}(x)&=\frac{t_{k+1}-t}{h_k}\nabla\hat{b}(t_k,x)\\
    &\quad+\frac{t-t_k}{h_k}\nabla\hat{b}(t_{k+1},F_{t_k\to t_{k+1}}(x))\cdot(I_d+h_k\nabla\hat{b}(t_k,x)),\\
    \frac{\dif^2}{\dif t^2}\nabla G_{t_k\to t}(x)&=\frac{1}{h_k}\left[\nabla\hat{b}(t_{k+1},F_{t_k\to t_{k+1}}(x))\cdot(I_d+h_k\nabla\hat{b}(t_k,x))-\nabla\hat{b}(t_k,x)\right].
\end{aligned}$$
So, by that $h_kL\lesssim 1$,
$$\begin{aligned}
    \left\Vert\nabla G_{t_k\to t}(x)\right\Vert_F&\lesssim 1,\\
    \left\Vert\frac{\dif}{\dif t}\nabla G_{t_k\to t}(x)\right\Vert_F&\lesssim L,\\
    \left\Vert\frac{\dif^2}{\dif t^2}\nabla G_{t_k\to t}(X_{t_k})\right\Vert_F&\lesssim L\left(\varepsilon_{2,k}(X_{t_k})+\left\Vert\nabla b(t_k,X_{t_k})\right\Vert_F\right)\\
    &\quad+\frac{L^{3/2}}{h_k}\Vert F_{t_k\to t_{k+1}}(X_{t_k})-X_{t_{k+1}}\Vert_F\\
    &\quad+\frac{1}{h_k}\int_{t_k}^{t_{k+1}}\left\Vert\frac{\dif}{\dif s}\nabla b(s,X_s)\right\Vert_F\dif s\\
    &\quad+\frac{1}{h_k}\varepsilon_{2,k}(X_{t_k})+\frac{1}{h_k}\varepsilon_{2,k+1}(X_{t_{k+1}}),
\end{aligned}$$
The constants omitted by the notation ``$\lesssim$" above is uniform for all $x$ (or $X_{t_k}$) and $t$. So,
$$\begin{aligned}
    \mathbb{E}\left[\text{tr}[D_{\text{dis}}]\right]
    &\lesssim h_k\int_{t_{k}}^{t_{k+1}}\Bigg\{\mathbb{E}\left[\left\Vert\frac{\dif^2}{\dif s^2}\nabla b(s,X_s)\right\Vert_F\right]\\
    &\quad+L\mathbb{E}\left[\left\Vert\frac{\dif}{\dif s}\nabla b(s,X_s)\right\Vert_F\right]\\
    &\quad+\mathbb{E}\left[\left\Vert\nabla b(s,X_s)\right\Vert_F\cdot\left\Vert\frac{\dif^2}{\dif s^2}\nabla G_{t_k\to s}(X_{t_k})\right\Vert_F\right]\Bigg\}\dif s\\
    &\lesssim h_k\left(\mathbb{E}\left[\varepsilon_{2,k}(X_{t_k})^2\right]^{1/2}+\mathbb{E}\left[\varepsilon_{2,k+1}(X_{t_{k+1}})^2\right]^{1/2}+L^{3/2}h_k\mathbb{E}[\varepsilon_{1,k}(X_{t_k})^2]^{1/2}\right)\\
    &\qquad\cdot\left(d\gamma(t)^{-2}+d^{1/2}\gamma(t)^{-1}\mathbb{E}\left[\Vert x_0-x_1\Vert^2\right]^{1/2}\right)\\
    &\quad+h_k^2\bigg[\mathbb{E}\left[\Vert x_0-x_1\Vert^6\right]^{1/2}d^{3/2}\bar{\gamma}_k^{-3}+d^3\bar{\gamma}_k^{-6}\\
    &\qquad+(L+d\bar{\gamma}_k^{-2})\left(\mathbb{E}\left[\Vert x_0-x_1\Vert^6\right]^{1/3}d\bar{\gamma}_k^{-2}+\bar{\gamma}_k^{-4}d^2\right)\bigg]
\end{aligned}$$

Therefore, by adding $D_{\text{bias}}$, $D_{\text{est}}$ and $D_\text{dis}$ together and integrating with respect to $t$, we can obtain that
$$\begin{aligned}
    &\qquad\int_{t_0}^{t_N}\mathbb{E}_{X_t\sim\rho(t)}\left[|\nabla\cdot\tilde{b}(t,X_t)-\nabla\cdot b(t,X_t)|\right]\dif t\\
    &\lesssim\sum_{k=0}^{N-1}h_k\left(\mathbb{E}\left[\varepsilon_{2,k}(X_{t_k})^2\right]^{1/2}+\mathbb{E}\left[\varepsilon_{2,k}(X_{t_{k+1}})^2\right]^{1/2}\right)\\
    &\quad+\sum_{k=0}^{N-1}h_kL^{1/2}\left(\mathbb{E}\left[\varepsilon_{1,k}(X_{t_k})^2\right]^{1/2}+\mathbb{E}\left[\varepsilon_{1,k}(X_{t_{k+1}})\right]\right)\\
    &\quad+\sum_{k=0}^{N-1}h_k^3\left(d^3\bar{\gamma}^{-6}+M^3\bar{\gamma}^{-4}\right).
\end{aligned}$$
where $M=\max\{\mathbb{E}\left[\Vert x_0-x_1\Vert^6\right]^{1/3},L,d\}$.

Combining the above discussions, we can get the bounds given in Theorem \ref{thm:heun}.

%% file: main.bbl
\begin{thebibliography}{}

\bibitem[Albergo et~al., 2023]{albergo2023interpolant}
Albergo, M.~S., Boffi, N.~M., and Vanden-Eijnden, E. (2023).
\newblock Stochastic interpolants: A unifying framework for flows and diffusions.

\bibitem[Albergo and Vanden-Eijnden, 2023]{albergo2023building}
Albergo, M.~S. and Vanden-Eijnden, E. (2023).
\newblock Building normalizing flows with stochastic interpolants.
\newblock In {\em The Eleventh International Conference on Learning Representations}.

\bibitem[Benton et~al., 2024a]{benton2024nearly}
Benton, J., Bortoli, V.~D., Doucet, A., and Deligiannidis, G. (2024a).
\newblock Nearly $d$-linear convergence bounds for diffusion models via stochastic localization.
\newblock In {\em The Twelfth International Conference on Learning Representations}.

\bibitem[Benton et~al., 2024b]{benton2024error}
Benton, J., Deligiannidis, G., and Doucet, A. (2024b).
\newblock Error bounds for flow matching methods.
\newblock {\em Transactions on Machine Learning Research}.

\bibitem[Chen et~al., 2023]{chen2023provably}
Chen, S., Chewi, S., Lee, H., Li, Y., Lu, J., and Salim, A. (2023).
\newblock The probability flow ode is provably fast.
\newblock In Oh, A., Naumann, T., Globerson, A., Saenko, K., Hardt, M., and Levine, S., editors, {\em Advances in Neural Information Processing Systems}, volume~36, pages 68552--68575. Curran Associates, Inc.

\bibitem[Grathwohl et~al., 2019]{grathwohl2018scalable}
Grathwohl, W., Chen, R. T.~Q., Bettencourt, J., and Duvenaud, D. (2019).
\newblock Scalable reversible generative models with free-form continuous dynamics.
\newblock In {\em International Conference on Learning Representations}.

\bibitem[Ho et~al., 2020]{ho2020ddpm}
Ho, J., Jain, A., and Abbeel, P. (2020).
\newblock Denoising diffusion probabilistic models.
\newblock In Larochelle, H., Ranzato, M., Hadsell, R., Balcan, M., and Lin, H., editors, {\em Advances in Neural Information Processing Systems}, volume~33, pages 6840--6851. Curran Associates, Inc.

\bibitem[Hongrui et~al., 2023]{chen2023improved}
Hongrui, C., Holden, L., and Jianfeng, L. (2023).
\newblock Improved analysis of score-based generative modeling: User-friendly bounds under minimal smoothness assumptions.
\newblock In {\em Proceedings of the 40th International Conference on Machine Learning}.

\bibitem[Huang et~al., 2025]{huang2025convergence}
Huang, D.~Z., Huang, J., and Lin, Z. (2025).
\newblock Convergence analysis of probability flow ode for score-based generative models.
\newblock {\em IEEE Transactions on Information Theory}, 71(6):4581–4601.

\bibitem[Li et~al., 2024a]{li2024accelerating}
Li, G., Huang, Y., Efimov, T., Wei, Y., Chi, Y., and Chen, Y. (2024a).
\newblock Accelerating convergence of score-based diffusion models, provably.
\newblock In {\em Forty-first International Conference on Machine Learning}.

\bibitem[Li et~al., 2024b]{li2024faster}
Li, G., Wei, Y., Chen, Y., and Chi, Y. (2024b).
\newblock Towards faster non-asymptotic convergence for diffusion-based generative models.

\bibitem[Li et~al., 2024c]{li2024sharp}
Li, G., Wei, Y., Chi, Y., and Chen, Y. (2024c).
\newblock A sharp convergence theory for the probability flow odes of diffusion models.

\bibitem[Li et~al., 2025]{li2025unified}
Li, R., Di, Q., and Gu, Q. (2025).
\newblock Unified convergence analysis for score-based diffusion models with deterministic samplers.
\newblock In {\em The Thirteenth International Conference on Learning Representations}.

\bibitem[Lipman et~al., 2023]{lipman2023flow}
Lipman, Y., Chen, R. T.~Q., Ben-Hamu, H., Nickel, M., and Le, M. (2023).
\newblock Flow matching for generative modeling.
\newblock In {\em The Eleventh International Conference on Learning Representations}.

\bibitem[Liu et~al., 2025]{liu2025finitetime}
Liu, Y., Chen, Y., Hu, R., and Huang, L. (2025).
\newblock Finite-time analysis of discrete-time stochastic interpolants.
\newblock In {\em Forty-second International Conference on Machine Learning}.

\bibitem[Song and Ermon, 2020]{song2020improved}
Song, Y. and Ermon, S. (2020).
\newblock Improved techniques for training score-based generative models.
\newblock In Larochelle, H., Ranzato, M., Hadsell, R., Balcan, M., and Lin, H., editors, {\em Advances in Neural Information Processing Systems}, volume~33, pages 12438--12448. Curran Associates, Inc.

\bibitem[Song et~al., 2021]{song2021scorebased}
Song, Y., Sohl-Dickstein, J., Kingma, D.~P., Kumar, A., Ermon, S., and Poole, B. (2021).
\newblock Score-based generative modeling through stochastic differential equations.
\newblock In {\em International Conference on Learning Representations}.

\end{thebibliography}
